\documentclass[sigconf]{acmart}
\usepackage{algorithm}

\usepackage{amssymb}
\usepackage{mathtools}
\usepackage{amsthm}
\usepackage{subcaption}
\newtheorem{remark}{Remark}
\usepackage{algorithmicx} \usepackage{algpseudocode}
\usepackage{pgfplots}
\pgfplotsset{compat=1.18}
\usepgfplotslibrary{groupplots}
\usetikzlibrary{calc}
\usepackage{subcaption}   
\AtBeginDocument{%
  }

\setcopyright{acmlicensed}
\copyrightyear{2026}
\acmYear{2026}
\setcopyright{cc}
\setcctype{by}
\acmConference[CAIS '26]{ACM Conference on AI and Agentic Systems}{May 26--29, 2026}{San Jose, CA, USA}
\acmBooktitle{ACM Conference on AI and Agentic Systems (CAIS '26), May 26--29, 2026, San Jose, CA, USA}
\acmDOI{10.1145/3786335.3813146}
\acmISBN{979-8-4007-2415-2/2026/05}

\begin{document}

\title{Echo: KV-Cache-Free Associative Recall with Spectral Koopman Operators}


\author{Anupama Sridhar}
\affiliation{%
  \institution{}
  \state{California}
  \country{USA}
  }
\orcid{0009-0008-5282-4024}

\author{Alexander Johansen}
\affiliation{%
  \institution{Stanford University}
  \department{Computer Science}
  \state{California}
  \country{USA}}
\email{arjo@stanford.edu}
\orcid{0000-0002-4993-7916}

\renewcommand{\shortauthors}{Sridhar and Johansen}

\begin{abstract}
Long chain-of-thought reasoning and agentic tool-calling produce traces spanning tens of thousands of tokens, yet Transformer KV caches grow linearly with sequence length, creating a memory bottleneck on commodity hardware.
State-space models offer constant-memory recurrence but suffer a \emph{memory cliff}: retrieval accuracy collapses once the gap between a stored fact and its query exceeds the effective horizon of the recurrent
state.
We introduce \textbf{Echo}, a KV-cache-free associative recall architecture built around \textbf{Spectral Koopman Attention (SKA)}; a drop-in
replacement for attention layers that augments SSM blocks with a closed-form dynamical operator whose sufficient statistics are accumulated in constant
memory with no KV cache. Echo fits a spectral linear system to the key and value history via kernel ridge regression and retrieves through a learned power-iterated filter, all from $O(r^{2})$ streaming state where $r$ is a small projection rank.
On the Multi-Query Associative Recall benchmark, a pure Mamba-2 SSM fails to exceed chance accuracy (${\sim}3\%$) across all gap lengths and KV-pair counts, while at the 50M parameter scale SKA-augmented models achieve $100\%$ retrieval accuracy on every configuration tested, including distractor gaps of $4{,}096$ tokens with
$32$ KV pairs.
Across five additional transfer benchmarks including needle-in-a-haystack, tool-trace, and multi-hop retrieval, SKA consistently outperforms both pure SSM and SSM+Attention hybrids while maintaining constant inference memory.
Ablations confirm that the spectral operator, not the prefix masking strategy, drives the retrieval gain.
\end{abstract}




\maketitle

\section{Introduction}

State space models such as Mamba~\cite{gu2023mamba} and Mamba-2~\cite{dao2024mamba2} offer linear-time inference and constant memory per decoding step, but compress the entire sequence into a fixed-size recurrent
state.
This compression is lossy: on associative recall benchmarks, pure SSMs fail to retrieve bindings across even moderate distractor gaps~\cite{arora2024zoology, jelassi2024repeat}, a failure mode we call the \emph{memory cliff}.

The standard remedy is attention. Hybrid architectures interleave SSM layers with causal attention blocks, restoring content-addressed lookup at the cost of a KV cache that grows linearly with context length, often dominating inference memory in production deployments.

We observe that content-addressed retrieval can be recast as kernel ridge regression whose sufficient statistics are accumulated as running sums in $O(r^{2}+rP)$ memory, eliminating the KV cache entirely.

We instantiate this in \textbf{Echo}, an architecture built around \textbf{Spectral Koopman Attention (SKA)}; a drop-in replacement for attention layers.
SKA accumulates key-value covariance statistics in a fixed-size streaming state, fits a whitened Koopman operator whose spectral modes identify which bindings to retain, and retrieves through a power-iterated filter.

On MQAR at the 50M parameter scale, Echo achieves perfect retrieval across all gap lengths and KV-pair counts where a pure Mamba-2 scores ${\sim}3\%$.
Across five transfer benchmarks spanning needle-in-a-haystack, tool-trace, and multi-hop retrieval, Echo outperforms both pure SSM and SSM+Attention hybrids while using fewer total parameters and constant inference memory.
We also test 180M parameter language modeling and find that Echo matches or exceeds performance on five of six zero-shot language modeling benchmarks compared to other 100M-400M parameter models. 

These results suggest that SSMs need not trade away their efficiency advantages
to recover the recall capabilities of attention.
\section{Related Works}
\subsection{State Space Models and Retrieval Limitations}
Structured state space models capture long-range dependencies through
continuous-time linear recurrences~\citep{gu2022s4, gupta2022diagonal,
smith2023s5}. Selective SSMs made these parameters input-dependent, enabling
content-based filtering with linear-time inference~\citep{gu2023mamba}, and
the SSD framework established a duality between SSMs and structured attention
for faster training~\citep{dao2024mamba2}.
However, the fixed-size recurrent state imposes well-documented retrieval
failures. On multi-query associative recall, a 70M-parameter attention model
outperforms a 1.4B-parameter gated-convolution model~\citep{arora2024zoology},
and formal analysis confirms that SSMs have strictly lower copying capacity
than Transformers~\citep{jelassi2024repeat}. Distant token relevance decays
exponentially~\citep{wang2024understanding}, and even adding a single attention
layer out of eight restores in-context learning capabilities that pure SSMs
lack~\citep{lieber2024jamba}.
\subsection{Hybrid Architectures and KV Cache Compaction}
These limitations have driven hybrid designs that interleave SSM and attention
layers~\citep{lieber2024jamba, hymba2024, waleffe2024empirical}, with
production deployments replacing up to 92\% of attention layers with SSM blocks
while retaining the remainder for retrieval~\citep{nemotronh2025}. A parallel
line of work reduces the KV cache through token eviction~\citep{zhang2023h2o,
snapkv2024}, quantization~\citep{liu2023kivi}, and dynamic
sparsification~\citep{dms2025neurips}.
Both approaches treat attention as the necessary retrieval mechanism and work to mitigate its memory cost.
Echo takes a different path: it replaces attention layers entirely with a constant-memory retrieval module (SKA), eliminating linear cache growth rather than compressing it.
\subsection{Koopman Operator Theory}
Koopman theory analyzes nonlinear dynamical systems through linear operators on observable functions~\citep{koopman1931hamiltonian, mezic2005spectral}, with machine learning applications in trajectory prediction via learned observable bases~\citep{lusch2018deep, williams2015data}.
SKA adapts the same least-squares estimation structure for a different purpose: associative retrieval within a sequence model, using persistent eigenmodes for content-addressed lookup rather than dynamical forecasting.
\subsection{In-Context Learning as Implicit Regression}
A line of work analyzes in-context learning in transformers as implicit execution of classical regression algorithms. 
\citet{akyurek2023learning} prove by construction that a transformer with constant depth and $O(d^2)$ hidden width can compute the closed-form ridge regression solution via iterative Sherman-Morrison rank-one updates, and empirically show that trained in-context learners behave like ridge predictors.
\citet{mahankali2024one} further show that a single layer of linear self-attention, trained on noisy linear regression, provably converges to one step of preconditioned gradient descent on the least-squares objective, and that this one-step predictor is Bayes-optimal under a Gaussian prior.
Recent work extends the training-dynamics picture to softmax attention, showing fast global convergence to the same ridge solution under suitable preconditioning~\citep{goel2026training}.
SKA takes the closed-form view and rather than unrolling Sherman-Morrison updates across transformer layers, or relying on gradient descent to discover the ridge solution implicitly, SKA accumulates the second-order sufficient statistics directly and solves the ridge system once per chunk via Cholesky factorization.
The Koopman spectral filter then adds a persistence-weighted readout that amplifies eigenmodes of the key dynamics with $|\lambda| \approx 1$ and suppresses transient modes.

Concurrent with our work, the test-time regression (TTR) framework~\citep{tumma2026preconditioned} unifies linear attention, DeltaNet~\citep{yang2024deltanet}, Gated DeltaNet~\citep{yang2025gated}, and MesaNet~\citep{vonoswald2025mesanet} as approximate solvers of the same online least-squares objective, differing only in the quality of their approximation. MesaNet solves this objective to optimality via conjugate gradient. SKA's value readout $B_v = C_v \tilde{G}^{-1}$ (Proposition~I.2) computes the same optimal solution in closed form via Cholesky factorization, avoiding CG's iterative convergence criterion and numerical instability. The Koopman power filter (\S\ref{app:power_filter}) then adds spectral selectivity with no analogue in the TTR hierarchy, amplifying persistent eigenmodes of the key dynamics while suppressing transient noise. This is a capability that the regression objective alone, however accurately solved, cannot express.
\section{Method}
\label{sec:method}

We present a hybrid architecture that replaces both softmax attention and SwiGLU feedforward layers with Koopman-operator-based alternatives, while retaining state-space models for local sequence processing.
Every component admits $O(1)$ inference state per layer, constant regardless of context length. The design rests on two observations: (1)~content-addressed retrieval can be recast as ridge regression whose sufficient statistics are accumulated as running sums in fixed memory, and (2)~the resulting computation is exactly the predictor that softmax attention learns to approximate through gradient descent on a nonconvex landscape~\cite{goel2026training}; but computed in closed form, with no KV cache.

\subsection{SSM Backbone}
\label{sec:ssm}

The majority of sequence layers use a selective state-space model for local, sequential processing.
For from-scratch language model training we use Mamba-2~\cite{dao2024mamba2}, which maintains a recurrent state $h_t \in \mathbb{R}^{N \times P}$ updated via a hardware-efficient selective scan:
\begin{equation}
    h_t = \bar{A}_t \, h_t + \bar{B}_t \, x_t,
    \qquad y_t = C_t \, h_t,
\end{equation}
where $\bar{A}_t$, $\bar{B}_t$, $C_t$ are input-dependent parameters.
For controlled synthetic experiments (\S\ref{sec:synthetic}) we use Mamba-2~\cite{lahoti2026mamba}, letting us isolate SKA's contribution on tasks that target SSM failure modes.

SSMs offer linear-time inference and constant memory per step, but their fixed-size recurrent state imposes a fundamental tradeoff: the multiplicative decay $\bar{A}_t$ attenuates all previously stored information at every step, causing retrieval accuracy to degrade exponentially with distance~\cite{wang2024understanding}.
This is the \emph{memory cliff}; not a capacity limitation but a structural consequence of lossy compression through multiplicative recurrence.

\subsection{Spectral Koopman Attention}
\label{sec:ska}
SKA replaces attention layers with a constant-memory retrieval module that estimates a compact linear dynamical model from running sufficient statistics and retrieves values at query positions, without materializing a $T \times T$ attention matrix or maintaining a KV-cache.

\subsubsection{Projection and Normalization}
\label{sec:ska-proj}
Learned linear maps produce per-head keys, queries, and values:
\begin{equation}
    z_t = W_k x_t \in \mathbb{R}^r, \quad
    z_t^q = W_q x_t \in \mathbb{R}^r, \quad
    v_t = W_v x_t \in \mathbb{R}^P,
\end{equation}
where $r$ is the Koopman rank and $P = d/H$ is the head dimension.
Key and query projections are initialized orthogonally to promote well-conditioned covariance matrices. SKA uses its own learned projections from the residual stream rather than sharing with Mamba's internal features: the SSM recurrence favors smooth, temporally correlated representations, while SKA needs discriminative inner products for content addressing.

Keys and queries are divided by a shared factor $m = \max_{s \le T} \|z_s\|_2$ (clamped at $10^{-6}$):
\begin{equation}
    \hat{z}_t = z_t / m, \qquad \hat{z}_t^q = z_t^q / m.
    \label{eq:key-norm}
\end{equation}
Unlike per-token $\ell_2$-normalization (which inflates low-norm noise tokens to unit norm), sequence-max normalization preserves the relative norm structure: high-norm fact-bearing tokens remain dominant while low-norm distractors are not amplified.
During autoregressive decoding, $m$ is frozen at the prefill value to prevent mixing normalization scales across accumulated statistics.\footnote{If needed, one could rescale all statistics by $(m_{\mathrm{old}}/m_{\mathrm{new}})^2$; we have not found this necessary.}

\subsubsection{Sufficient Statistic Accumulation}
\label{sec:ska-stats}
SKA estimates its operators from three sufficient statistics: a regularized Gram matrix $G$, a cross-temporal covariance $M$, and a value-key covariance $C_v$.
The central insight is that these statistics are \emph{additive}: each new token contributes $\hat{z}_t\hat{z}_t^\top$ to $G$, $v_t\hat{z}_t^\top$ to $C_v$, and $\hat{z}_t\hat{z}_{t-1}^\top$ to $M$ without modifying any previously accumulated contribution.
The ridge regression solution $B^* = C_v \tilde{G}^{-1}$ depends on the data \emph{only} through these second-order statistics, which are exact regardless of how many tokens contributed to them.

This is why SKA achieves $O(1)$ memory without the memory cliff: the fixed-size state is not a lossy bottleneck but an \emph{exact sufficient statistic} for the retrieval computation.
Longer sequences produce better-conditioned estimates---the perturbation bound (Theorem~C.1) tightens as $\lambda_{\min}(\tilde{G})$ grows---rather than degrading retrieval accuracy. This stands in contrast to SSM recurrence, where multiplicative decay progressively erases stored bindings.

How the statistics are assembled depends on the inference setting; the operator estimation (\S\ref{sec:ska-operator}) is identical in all cases.

\paragraph{Prefix mode (agentic inference).}
When processing a long prefix (tool output, retrieved documents, multi-turn context) before generating, the statistics are accumulated over the entire prefix with no chunking:
\begin{align}
    G &= \textstyle\sum_{t=1}^{T_{\mathrm{pre}}}
         \hat{z}_t \hat{z}_t^\top + \varepsilon I_r,
         \label{eq:gram} \\
    M &= \textstyle\sum_{t=2}^{T_{\mathrm{pre}}}
         \hat{z}_t \hat{z}_{t-1}^\top,
         \label{eq:cross} \\
    C_v &= \textstyle\sum_{t=1}^{T_{\mathrm{pre}}}
         v_t \hat{z}_t^\top,
         \label{eq:value-cov}
\end{align}
where $\varepsilon = 10^{-3}$ is ridge regularization. The operator is estimated once; subsequent tokens update the accumulators incrementally (\S\ref{sec:recurrent}). Every prefix token contributes to the operator serving the first generated token, with no chunk-boundary latency.

\paragraph{Masked mode (retrieval).}
For bidirectional encoding (e.g., dense retrieval), the sums in Eqs.~\ref{eq:gram}--\ref{eq:value-cov} run over all unmasked positions, yielding a single operator from the full visible context.

\paragraph{Chunk-causal mode (LM training).}
Causal LM training requires each position to attend only to earlier positions.
Computing $T$ separate operator estimates would be prohibitive, so we partition the sequence into $C = \lceil T/S \rceil$ non-overlapping chunks of size $S$ and estimate one operator per chunk from strictly preceding chunks.
Per-chunk statistics are computed in parallel:
\begin{equation}
    G_c = \!\sum_{t \in \text{chunk } c}\!
        \hat{z}_t \hat{z}_t^\top, \quad
    M_c = \!\!\sum_{\substack{t \in \text{chunk } c \\
        t > \text{first}(c)}}\!\!
        \hat{z}_t \hat{z}_{t-1}^\top, \quad
    C_{v,c} = \!\sum_{t \in \text{chunk } c}\!
        v_t \hat{z}_t^\top.
    \label{eq:per-chunk}
\end{equation}
Accumulated prefix statistics available to queries in chunk $c$ are:
\begin{align}
    G^{(c)} &= \varepsilon I_r +
        \textstyle\sum_{i=0}^{c-1} G_i,
        \label{eq:gram-prefix} \\
    M^{(c)} &= \textstyle\sum_{i=0}^{c-1} M_i
        + \sum_{i=1}^{c-1}
        \hat{z}_{i,0}\,\hat{z}_{i-1,-1}^\top,
        \label{eq:cross-prefix} \\
    C_v^{(c)} &= \textstyle\sum_{i=0}^{c-1} C_{v,i},
        \label{eq:value-cov-prefix}
\end{align}
computed via \texttt{cumsum} with a one-position shift (exclusive prefix sum).
The second sum in Eq.~\ref{eq:cross-prefix} captures boundary transitions between adjacent chunks.
The default $S = 64$ balances causal granularity against the number of $O(r^3)$ Cholesky factorizations.
In all three modes the downstream estimation is identical, so we drop the superscript $(c)$ below.

\subsubsection{Operator Estimation, Retrieval, and Connection to
Attention}
\label{sec:ska-operator}

The Gram matrix is factored as $G = LL^\top$ via batched Cholesky decomposition (with jittered fallback $G + 10^{-4}I$ on failure).
All covariance accumulation and solves are in FP32.

The Koopman transition and value readout operators are:
\begin{equation}
    A_w = M G^{-1}, \qquad B_v = C_v G^{-1},
    \label{eq:operators}
\end{equation}
computed via Cholesky solves.

\paragraph{Connection to attention.}
The value readout $B_v = C_v \tilde{G}^{-1}$ is the unique minimizer of the regularized least-squares objective $\min_B \sum_t \|v_t - B\hat{z}_t\|^2 + \varepsilon\|B\|_F^2$ (Proposition~I.2).
This is the natural objective for associative retrieval: predict each value from its key with minimum error.
Recent theoretical work~\cite{goel2026training} shows that softmax attention, when trained to convergence, implicitly solves this same objective---its global optimum produces outputs approximating $C_v \tilde{G}^{-1} q$.
Linear attention computes only $C_v q$, which corresponds to a single gradient step from zero on the same objective and omits the Gram inverse that decorrelates correlated keys (Appendix~I).
SKA computes the full ridge solution directly via Cholesky factorization in $O(r^3)$, bypassing the nonconvex optimization landscape that attention must navigate through gradient descent.

\paragraph{Spectral normalization.}
To ensure stability, $A_w$ is spectrally normalized by its leading singular value (6 power iterations, detached):
\begin{equation}
    A_w \leftarrow \frac{\gamma}{\max(\sigma_{\max}(A_w), 1)} A_w,
\end{equation}
where $\gamma \in [1.0, 1.5]$ is learnable. This bounds the spectral radius, preventing unbounded amplification while allowing gradients through $A_w$ (straight-through estimation).

\paragraph{Power spectral filter and readout.}
The Koopman transition operator $A_w$ adds spectral selectivity with no analogue in standard attention: its eigenvalues identify which key-value bindings are \emph{persistent} ($|\lambda_i| \approx 1$) versus \emph{transient} ($|\lambda_i| \ll 1$).
Rather than eigendecomposition, persistent modes are amplified by repeated application.
Given whitened query $w_q = L^{-1}\hat{z}^q$, the output is:
\begin{equation}
    w_f = A_w^K w_q, \qquad z_f = L\,w_f, \qquad
    \hat{y} = B_v z_f \in \mathbb{R}^P,
\end{equation}
where $K{=}2$ controls selectivity: modes with $|\lambda_i| \ll 1$ are exponentially suppressed.
The output is scaled by learnable $\eta$ (init 1.5) and projected to model dimension via a zero-initialized linear layer, so SKA begins as a near-identity perturbation of the residual stream.
When available, the post-Cholesky chain is fused into a single Triton kernel that loads $A_w$, $L$, $B_v$, $w_q$ into SRAM for $K+2$ matmuls without HBM round-trips.

When the power filter is disabled ($K{=}0$), the output reduces to $\eta \cdot C_v \tilde{G}^{-1} \hat{z}^q$---exactly the ridge regression predictor.
The power filter with $K \geq 1$ adds spectral weighting that amplifies persistent bindings and suppresses noise, going beyond what attention's implicit optimization recovers.

\paragraph{Gradient flow.}
The closed-form computation also benefits training.
The SKA Jacobian with respect to the query is $J_{\mathrm{SKA}} = B_v \cdot L \cdot A_w^K \cdot L^{-1}$, a product of full-rank matrices by construction.
In contrast, each softmax attention layer's backward pass compresses gradients by its entropy-dependent effective rank~\cite{godey2026lost}.
Critically, these compressions \emph{compound across depth}: a gradient passing through $n_a$ attention layers suffers $n_a$ successive rank bottlenecks, and since the rank of a matrix product is bounded by the minimum rank of any factor, even a single low-entropy attention layer can throttle gradient flow to all preceding layers.
In Echo the gradient encounters zero softmax compressions beyond the LM head: SKA contributes a full-rank $r \times r$ linear system, and the Koopman MLP contributes a norm-preserving orthogonal rotation (\S\ref{sec:koopman-mlp}).
The only rank bottleneck in the entire backward pass is the LM head itself.
This multiplicative preservation of gradient rank across depth---not just per-layer efficiency---may explain the data efficiency observed in our experiments (Table~\ref{tab:benchmark}), where 10B tokens suffice to match models trained on 100B (Appendix~J).

\subsection{Spectral Koopman MLP}
\label{sec:koopman-mlp}

SwiGLU~\cite{shazeer2020glu} uses three $d \times d_{\mathrm{ff}}$ matrices ($3\,d\,d_{\mathrm{ff}}$ parameters).
We replace it with a feedforward layer that applies one step of a spectrally-constrained dynamical system, requiring only two matrices ($2\,d\,d_k$ parameters, a $\tfrac{1}{3}$ reduction).

\paragraph{Lift.}
The layer-normed input is projected and activated: $g = \mathrm{SiLU}(W_{\mathrm{lift}}\,h) \in \mathbb{R}^{d_k}$, where $d_k = \lceil d \cdot e / 64 \rceil \times 64$ with $e \approx 8/3$.

\paragraph{Block-diagonal rotation.}
The vector $g$ is reshaped into $d_k/2$ pairs and evolved by learnable complex eigenvalues $\lambda_i = \gamma_i + i\omega_i$:
\begin{equation}
    \begin{pmatrix} z_1^{(i)} \\ z_2^{(i)} \end{pmatrix}
    =
    \begin{pmatrix}
        \gamma_i & \omega_i \\
        -\omega_i & \gamma_i
    \end{pmatrix}
    \begin{pmatrix} g_1^{(i)} \\ g_2^{(i)} \end{pmatrix}.
\end{equation}
The eigenvalue modulus is clamped to the unit disk ($|\lambda_i| \le 1$), the pointwise analogue of the spectral normalization applied to $A_w$ in SKA.

\paragraph{Readout.}
The rotated pairs are interleaved back and projected: $\mathrm{output} = x + W_{\mathrm{readout}}\,z$. An optional gated variant adds $\sigma(W_{\mathrm{gate}}\,h)$ elementwise before readout, bringing parameters to SwiGLU parity while retaining spectral structure.

\paragraph{Gradient preservation.}
The block-diagonal rotation matrix $R$ is orthogonal: $\|R\|_2 = 1$, $\sigma_{\min}(R) = \sigma_{\max}(R) = 1$.
It is exactly norm-preserving and rank-preserving in the backward pass.
Combined with SKA's full-rank Jacobian, Echo introduces \emph{zero} additional rank compression beyond the LM head across the entire network depth.
In contrast, SwiGLU's backward pass flows through three unconstrained weight matrices interleaved with a gating nonlinearity, where gradient magnitude depends on the learned weights without structural guarantees on conditioning.

\subsection{Architecture and Baselines}
\label{sec:architecture}

\paragraph{Echo (Mamba-2 + SKA + Koopman MLP).}
Echo follows the Nemotron-H~\cite{nemotronh2025} hybrid layout: $N$ blocks, each a sequence layer followed by a feedforward layer.
The majority of sequence layers are Mamba-2; a fraction are replaced with evenly-spaced SKA layers. All feedforward layers use the Spectral Koopman MLP. The first and last layers are always Mamba-2.

For the 180M model: $d{=}768$, $N{=}24$, 2 SKA layers at positions $\{8,16\}$, $r{=}48$, $H{=}12$. For the 50M model: $d{=}448$, $N{=}16$, 4 SKA layers at $\{3,7,11,15\}$, $r{=}56$, $H{=}7$, $S{=}64$, $d_{\mathrm{state}}{=}64$.

\paragraph{Baselines.}
To isolate the contribution of each component we compare three variants at 180M, all sharing the same $d$, $N$, and layer layout: (1)~\emph{Mamba-2 only + SwiGLU}: all sequence layers are Mamba-2 with no global retrieval, and all feedforward layers are SwiGLU; (2)~\emph{Mamba-2 + SKA + SwiGLU}: SKA layers are inserted at the same positions as Echo but feedforward layers remain SwiGLU, isolating SKA's retrieval contribution; (3)~Echo described above.
For retrieval-focused evaluations (NIAH, MQAR, inverse matching) we primarily report results from the Mamba-2 + SKA configurations, as these most directly test the global retrieval mechanism.
For synthetic experiments targeting SSM failure modes we use Mamba-2 as the backbone with SKA at the final layers, isolating retrieval in a controlled setting.

\subsection{O(1) Recurrent Inference}
\label{sec:recurrent}

Every component admits $O(1)$ state during autoregressive generation:
\begin{itemize}
    \item \textbf{Mamba-2/3:} conv state
        ($d_{\mathrm{inner}} \times d_{\mathrm{conv}}$) plus SSM state
        ($n_{\mathrm{heads}} \times d_{\mathrm{head}} \times
        d_{\mathrm{state}}$), both fixed.
    \item \textbf{SKA:} accumulators
        $G \in \mathbb{R}^{r \times r}$,
        $M \in \mathbb{R}^{r \times r}$,
        $C_v \in \mathbb{R}^{P \times r}$, previous key
        $z_{t-1} \in \mathbb{R}^r$, and scalar max norm, totaling
        $2r^2 + Pr + r + 1$ floats per head, independent of $T$. Each
        step updates
        $G \leftarrow G + \hat{z}_t\hat{z}_t^\top$,
        $M \leftarrow M + \hat{z}_t\hat{z}_{t-1}^\top$,
        $C_v \leftarrow C_v + v_t\hat{z}_t^\top$ and re-estimates the
        operator.
    \item \textbf{Koopman MLP:} stateless (pointwise).
\end{itemize}
For the 50M model the total SKA state is ${\approx}77$\,KB in FP32.
Combined with Mamba-2, total recurrent state is on the order of hundreds of kilobytes, enabling generation over arbitrarily long contexts with bounded memory.

\subsection{Complexity Summary}
\label{sec:complexity}

\begin{table}[h]
\centering
\small
\caption{Per-layer complexity. $T$: sequence length, $d$: model dim,
$r$: SKA rank, $S$: chunk size.}
\label{tab:complexity}
\begin{tabular}{lcc}
\toprule
\textbf{Component} & \textbf{Train (parallel)}
    & \textbf{Infer (per step)} \\
\midrule
Causal Attention & $O(T^2 d)$ & $O(Td)$ \\
Mamba-2 & $O(Td)$ & $O(d)$ \\
SKA & $O(Tr^2/S + r^3)$ & $O(r^3)$ \\
SwiGLU MLP & $O(Td\,d_{\mathrm{ff}})$
    & $O(d\,d_{\mathrm{ff}})$ \\
Koopman MLP & $O(Td\,d_k)$ & $O(d\,d_k)$ \\
\bottomrule
\end{tabular}
\end{table}

During training, SKA's cost is dominated by the $O(r^3)$ Cholesky per chunk per head, amortized over $S$ tokens; since $r \ll T$ the total is sublinear in $T$.
At inference each step requires an $O(r^3)$ solve, negligible for $r \le 64$. The Koopman MLP saves $\tfrac{1}{3}$ of parameters relative to SwiGLU at identical FLOPs per token.
\section{Experimental Setup}
\label{sec:experiments}

We evaluate Echo in three regimes: controlled synthetic tasks that isolate SSM retrieval failures (\S\ref{sec:synthetic}), MQAR fine-tuning at the 50M scale (\S\ref{sec:mqar-finetune}), and from-scratch language modeling at 180M parameters (\S\ref{sec:lm-setup}).

\subsection{Synthetic Transfer Experiments}
\label{sec:synthetic}

\paragraph{Models.}
We compare three architectures trained from scratch with identical optimization (AdamW, $\text{lr}=3\times10^{-4}$, 6{,}000 steps, batch size 16): \textbf{SSM} (4-layer Mamba-2, 1.11M params), \textbf{SSM+Attn} (2 Mamba-2 + 2 causal attention layers with RoPE, 998K params), and \textbf{SSM+SKA} (2 Mamba-2 + 2 SKA layers, rank $r{=}24$, 982K params; the smallest model).
All share $d{=}128$, 4 heads, $d_{\text{state}}{=}16$, $V{=}128$, and SwiGLU MLPs.

\paragraph{Protocol.}
Each model is trained once on a mixed dataset of system-prompt and tool-trace examples, then evaluated zero-shot on five held-out retrieval tasks it was never trained on: system prompt recall, tool trace retrieval, needle-in-a-haystack, multi-hop composition, and common word identification.
Each task varies KV pairs and sequence length to produce accuracy grids.
We additionally evaluate length generalization at $2$--$64{\times}$ beyond training sequence length, and verify that recurrent inference (token-by-token accumulation of Koopman statistics) preserves accuracy relative to the parallel forward pass.

\subsection{MQAR (50M Scale)}
\label{sec:mqar-finetune}

To test retrieval at a more realistic parameter budget, we train 50M-parameter variants of Mamba-2, Mamba-2+Attention, and Mamba-2+SKA on the multi-query associative recall benchmark~\citep{arora2024zoology}.
The MQAR grid varies KV pair count $M \in \{4, 8, 16, 32\}$ and distractor gap length from 64 to 4{,}096 tokens.
Each cell is trained and evaluated independently (in-task), testing whether retrieval scales with model size across the full difficulty grid.

\subsection{Language Modeling (180M Scale)}
\label{sec:lm-setup}

\paragraph{Models.}
We train 50M- and 180M-parameter \textbf{Echo} models from scratch: \textbf{Mamba-2 + SKA + Koopman MLP} (our method) with $d{=}768$, 24 layers, 2 SKA layers at positions $\{8, 16\}$, rank $r{=}48$, 12 heads, and Spectral Koopman MLP at every layer.

\paragraph{Training.}
Both models are trained on FineWeb-Edu~\citep{penedo2024fineweb} (sample-10BT) for ${\sim}10$B tokens (50{,}000 steps, sequence length 2{,}048, effective batch size 96) using DeepSpeed ZeRO Stage 1 with BF16 on a single NVIDIA B200.

\paragraph{Evaluation.}
We report utilize the LM Evaluation Harness WikiText-103 perplexity~\cite{merity2016pointer}, standard benchmarks via lm-evaluation-harness (HellaSwag~\cite{zellers2019hellaswag}, PIQA~\cite{bisk2020piqa}, ARC~\cite{clark2018think}, WinoGrande~\cite{sakaguchi2021winogrande}, LAMBADA~\cite{paperno2016lambada}), needle-in-a-haystack retrieval at context lengths up to 4{,}096, and peak memory during autoregressive generation to verify the $O(1)$ inference state property.
\section{Results}
\subsection{Transfer Retrieval (Sub-Million Scale)}

To test whether SKA's retrieval mechanism generalizes across task families, we trained each model once on a mixed system-prompt and tool-trace curriculum and evaluated zero-shot on five held-out retrieval tasks.
SSM+SKA (982K params) achieved the highest mean accuracy on all five benchmarks despite being the smallest model, scoring 85\% on system prompt, 81\% on tool trace, 84\% on needle-in-a-haystack, 77\% on multi-hop, and 74\% on common word retrieval (Table~\ref{tab:synthetic-tasks}).
The pure SSM (1.11M params) scored 64\%, 57\%, 52\%, 39\%, and 53\% respectively, with accuracy degrading sharply at longer sequence lengths consistent with exponential memory decay~\citep{wang2024understanding}.
SSM+Attention (998K params) fell between the two on every task.

To determine whether this advantage persists beyond the training distribution, we evaluated all models at sequence lengths $2$--$64{\times}$ longer than those seen during training (Table~\ref{tab:length-gen}).
At $4{,}096$ tokens, SSM+SKA retained 65\% accuracy while SSM dropped to 2\% and SSM+Attn to 5\%.
The flat accuracy profile of SSM+SKA is consistent with a retrieval mechanism whose sufficient statistics are length-independent.

%
%
\begin{figure*}[t]
    \centering
    \textbf{Zero-Shot Retrieval Transfer Across Architectures}\par\medskip

    \begin{subfigure}[t]{1.0\linewidth}
        \centering
\begin{tikzpicture}
  \definecolor{cSSM}{HTML}{0072B2}     
  \definecolor{cAttn}{HTML}{D55E00}    
  \definecolor{cSKA}{HTML}{009E73}     

  \pgfplotsset{
    every axis/.style={
      width=0.32\linewidth, height=4.2cm,
      xmode=log, log basis x=2,
      xtick={64,128,256,512,1024,2048},
      xticklabels={64,128,256,512,1024,2048},
      xticklabel style={font=\footnotesize},
      yticklabel style={font=\footnotesize},
      xlabel={\footnotesize Sequence length},
      ymin=0, ymax=1.02,
      ytick={0,0.2,0.4,0.6,0.8,1.0},
      grid=major,
      grid style={dotted, gray!45, line width=0.4pt},
      major tick length=2.5pt,
      axis line style={line width=0.6pt},
      tick align=outside,
      title style={font=\footnotesize\bfseries, yshift=-2pt},
      every axis plot/.append style={line width=0.9pt, mark size=1.6pt},
    },
  }

  \begin{groupplot}[
    group style={
      group size=3 by 1,
      horizontal sep=10pt,
      ylabels at=edge left,
      yticklabels at=edge left,
    },
    ylabel={\footnotesize Accuracy},
  ]
    \nextgroupplot[title={$\mathit{KV}=1$}]
      \addplot[color=cSSM,  dashed,     mark=*]
        coordinates {(64,0.750)(128,0.755)(256,0.700)(512,0.660)(1024,0.510)(2048,0.280)};
      \addplot[color=cAttn, dashdotted, mark=square*]
        coordinates {(64,0.925)(128,0.910)(256,0.900)(512,0.900)(1024,0.880)(2048,0.660)};
      \addplot[color=cSKA,  solid,      mark=diamond*]
        coordinates {(64,0.925)(128,0.935)(256,0.950)(512,0.880)(1024,0.860)(2048,0.725)};

    \nextgroupplot[title={$\mathit{KV}=4$}]
      \addplot[color=cSSM,  dashed,     mark=*]
        coordinates {(64,0.640)(128,0.620)(256,0.640)(512,0.550)(1024,0.470)(2048,0.235)};
      \addplot[color=cAttn, dashdotted, mark=square*]
        coordinates {(64,0.880)(128,0.890)(256,0.850)(512,0.890)(1024,0.755)(2048,0.595)};
      \addplot[color=cSKA,  solid,      mark=diamond*]
        coordinates {(64,0.860)(128,0.890)(256,0.885)(512,0.875)(1024,0.810)(2048,0.680)};

    \nextgroupplot[
      title={$\mathit{KV}=8$},
      legend to name=niahlegend,
      legend columns=3,
      legend style={
        /tikz/every even column/.append style={column sep=10pt},
        draw=none, font=\footnotesize,
      },
      legend image post style={line width=0.9pt},
    ]
      \addplot[color=cSSM,  dashed,     mark=*]
        coordinates {(64,0.545)(128,0.480)(256,0.485)(512,0.430)(1024,0.260)(2048,0.040)};
      \addlegendentry{SSM (1.11M)}
      \addplot[color=cAttn, dashdotted, mark=square*]
        coordinates {(64,0.810)(128,0.800)(256,0.780)(512,0.750)(1024,0.710)(2048,0.510)};
      \addlegendentry{SSM+Attn (998K)}
      \addplot[color=cSKA,  solid,      mark=diamond*]
        coordinates {(64,0.860)(128,0.840)(256,0.835)(512,0.785)(1024,0.760)(2048,0.625)};
      \addlegendentry{SSM+SKA (982K)}
  \end{groupplot}

  \node[anchor=north, yshift=-22pt]
    at ($(group c2r1.south)+(0,-1.0em)$) {\pgfplotslegendfromname{niahlegend}};
\end{tikzpicture}
        \caption{Needle-in-a-haystack}
        \label{fig:niah}
    \end{subfigure}

    \vspace{0.5em}

    \begin{subfigure}[t]{1.0\linewidth}
        \centering
\begin{tikzpicture}
  \definecolor{cSSM}{HTML}{0072B2}     
  \definecolor{cAttn}{HTML}{D55E00}    
  \definecolor{cSKA}{HTML}{009E73}     

  \pgfplotsset{
    every axis/.style={
      width=0.32\linewidth, height=4.2cm,
      xmode=log, log basis x=2,
      xtick={16,32,64,128,256},
      xticklabels={16,32,64,128,256},
      xticklabel style={font=\footnotesize},
      yticklabel style={font=\footnotesize},
      xlabel={\footnotesize Sequence length},
      ymin=0, ymax=1.02,
      ytick={0,0.2,0.4,0.6,0.8,1.0},
      grid=major,
      grid style={dotted, gray!45, line width=0.4pt},
      major tick length=2.5pt,
      axis line style={line width=0.6pt},
      tick align=outside,
      title style={font=\footnotesize\bfseries, yshift=-2pt},
      every axis plot/.append style={line width=0.9pt, mark size=1.6pt},
    },
  }

  \begin{groupplot}[
    group style={
      group size=3 by 1,
      horizontal sep=10pt,
      ylabels at=edge left,
      yticklabels at=edge left,
    },
    ylabel={\footnotesize Accuracy},
  ]
    \nextgroupplot[title={$\mathit{KV}=3$}]
      \addplot[color=cSSM,  dashed,     mark=*]
        coordinates {(16,0.530)(32,0.475)(64,0.520)(128,0.465)(256,0.385)};
      \addplot[color=cAttn, dashdotted, mark=square*]
        coordinates {(16,0.760)(32,0.760)(64,0.800)(128,0.760)(256,0.685)};
      \addplot[color=cSKA,  solid,      mark=diamond*]
        coordinates {(16,0.840)(32,0.840)(64,0.820)(128,0.805)(256,0.735)};

    \nextgroupplot[title={$\mathit{KV}=4$}]
      \addplot[color=cSSM,  dashed,     mark=*]
        coordinates {(16,0.575)(32,0.475)(64,0.455)(128,0.395)(256,0.325)};
      \addplot[color=cAttn, dashdotted, mark=square*]
        coordinates {(16,0.740)(32,0.745)(64,0.765)(128,0.760)(256,0.660)};
      \addplot[color=cSKA,  solid,      mark=diamond*]
        coordinates {(16,0.835)(32,0.795)(64,0.790)(128,0.745)(256,0.740)};

    \nextgroupplot[
      title={$\mathit{KV}=8$},
      legend to name=multihoplegend,
      legend columns=3,
      legend style={
        /tikz/every even column/.append style={column sep=10pt},
        draw=none, font=\footnotesize,
      },
      legend image post style={line width=0.9pt},
    ]
      \addplot[color=cSSM,  dashed,     mark=*]
        coordinates {(16,0.365)(32,0.410)(64,0.285)(128,0.180)(256,0.005)};
      \addlegendentry{SSM (1.11M)}
      \addplot[color=cAttn, dashdotted, mark=square*]
        coordinates {(16,0.690)(32,0.685)(64,0.710)(128,0.625)(256,0.425)};
      \addlegendentry{SSM+Attn (998K)}
      \addplot[color=cSKA,  solid,      mark=diamond*]
        coordinates {(16,0.770)(32,0.755)(64,0.770)(128,0.665)(256,0.555)};
      \addlegendentry{SSM+SKA (982K)}
  \end{groupplot}

  \node[anchor=north, yshift=-22pt]
    at ($(group c2r1.south)+(0,-1.0em)$) {\pgfplotslegendfromname{multihoplegend}};
\end{tikzpicture}
        \caption{Two-hop retrieval}
        \label{fig:multihop}
    \end{subfigure}

    \caption{Zero-shot retrieval accuracy vs.\ sequence length and KV count
    for single-hop (a) and two-hop (b) tasks. The pure SSM degrades sharply
    at longer sequences; SSM+Attn improves but still drops; SSM+SKA (982K
    params) maintains the highest accuracy and flattest profile across all
    settings.}
    \label{fig:transfer-retrieval}
\end{figure*}
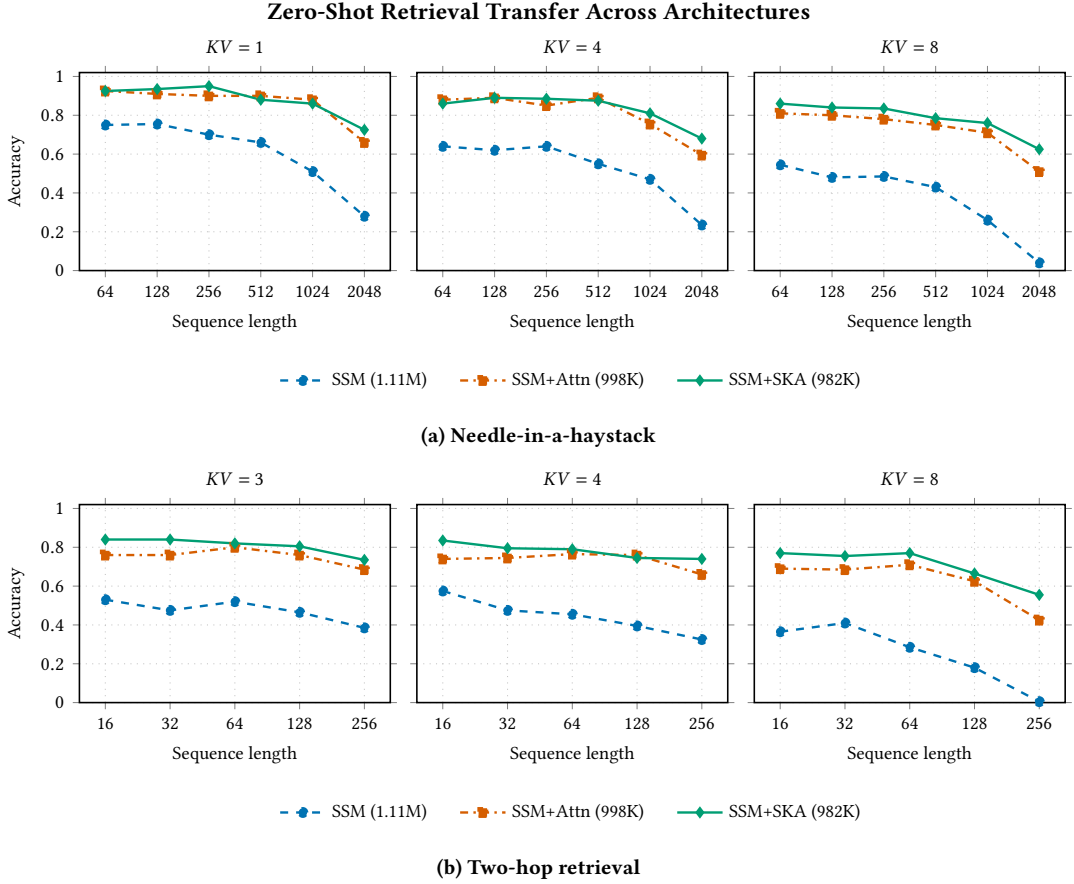
\begin{table}[t]
\centering
\caption{Length generalization on needle-in-a-haystack retrieval (KV=1). Models were trained at sequence length 64 and evaluated at up to $64{\times}$ longer sequences. SSM+SKA retains 65\% accuracy at 4{,}096 tokens where both baselines collapse to near chance.}
\label{tab:length-gen}
\begin{tabular}{rccc}
\toprule
\textbf{Seq Length} & \textbf{SSM} & \textbf{SSM+Attn} & \textbf{SSM+SKA} \\
\midrule
64 (train) & 88.6 & 88.5 & \textbf{94.5} \\
128        & 82.1 & \textbf{97.4} & 89.4 \\
256        & 75.8 & 92.9 & \textbf{93.3} \\
512        & 66.9 & 78.1 & \textbf{95.6} \\
1024       & 37.3 & 64.8 & \textbf{86.4} \\
2048       &  2.0 & 42.5 & \textbf{81.2} \\
4096       &  2.0 &  5.0 & \textbf{65.1} \\
\bottomrule
\end{tabular}
\end{table}
\begin{table}[t]
\centering
\caption{Mean accuracy across all sequence lengths and KV counts on six synthetic retrieval tasks. Models were trained on a mixed system-prompt and tool-trace curriculum and evaluated zero-shot on all tasks. SSM+SKA achieves the highest accuracy on every task despite having the fewest parameters.}
\label{tab:synthetic-tasks}
\begin{tabular}{rccc}
\toprule
\textbf{Task} & \textbf{SSM} & \textbf{SSM+Attn} & \textbf{SSM+SKA} \\
\midrule
MQAR       & 55.1 & 76.0 & \textbf{82.2} \\
SysPrompt  & 64.2 & 81.3 & \textbf{85.3} \\
ToolTrace  & 56.5 & 76.2 & \textbf{80.5} \\
NIAH       & 52.0 & 81.1 & \textbf{84.2} \\
MultiHop   & 39.3 & 71.1 & \textbf{76.7} \\
CommonWord & 52.9 & 68.2 & \textbf{74.1} \\
\midrule
\textbf{Mean} & 54.4 & 76.2 & \textbf{81.1} \\
\bottomrule
\end{tabular}
\end{table}

\subsection{MQAR Fine-Tuning (50M Scale)}
Having established SKA's advantage at sub-million scale, we tested whether the memory cliff persists at a more realistic parameter budget by fine-tuning 50M-parameter models on MQAR.
We tested a grid of MQAR configurations (KV pairs $M \in \{4, 8, 16, 32\}$, distractor gaps 64--4{,}096 tokens), trained and evaluated independently.

A pure Mamba-2 scored ${\sim}3\%$ on every configuration, failing to exceed chance regardless of gap length or KV count.
Mamba-2+Attention achieved 100\% on all but one cell, degrading at the largest gap with $M{=}4$, perhaps due to lag of training signal.
Mamba-2+SKA achieved \textbf{100\% on every cell tested}.
This confirms that the memory cliff is a structural limitation of SSM recurrence, not a capacity issue, and that SKA eliminates it entirely at this scale.

\subsection{Language Modeling (180M Scale)}
\begin{table*}[t]
\centering
\caption{Zero-shot benchmark comparison of Echo-LM against small language models (180M parameters).
All 180M models in the top section use comparable parameter budgets; the key differentiators are architecture and training data.}
\label{tab:benchmark}
\resizebox{\textwidth}{!}{%
\begin{tabular}{lrrccccccc}
\toprule
\textbf{Model} & \textbf{Params} & \textbf{Tokens} &
\textbf{FW ppl$\downarrow$} & \textbf{HS$\uparrow$} & \textbf{PIQA$\uparrow$} &
\textbf{ARC-E$\uparrow$} & \textbf{ARC-C$\uparrow$} &
\textbf{WG$\uparrow$} & \textbf{LMB$\uparrow$} \\
\midrule
Echo-180M & 180M & \textbf{10B} & 16.48 & \textbf{44.3} & \textbf{70.9} & \textbf{58.8} & \textbf{34.2} & \textbf{55.4}
& \textbf{39.8} \\
Echo-50M & 50M & 3B & 27.48 & 29.3 & 59.5 & 42.0 & 24.4 & 49.5 &
17.1 \\
\midrule
\multicolumn{10}{l}{\textit{180M-class models (matched scale~\cite{lahoti2026mamba})}} \\
\midrule
Transformer-180M & 180M & 100B & 16.89 & 39.0 & 67.1 & 59.8 & 27.9 & 51.2
& 32.5 \\
GDN-180M & 180M & 100B & 16.52 & 40.2 & 66.3 & \textbf{62.3} & 28.2 & 51.7
& 31.3 \\
Mamba-2-180M & 180M & 100B & 16.76 & 40.1 & 66.8 & 60.1 & 27.3 & 52.0
& 30.9 \\
Mamba-3-SISO-180M & 180M & 100B & 16.59 & 40.8 & 66.1 & 61.5 & 27.9 & 52.0
& 32.5 \\
Mamba-3-MIMO-180M & 180M & 100B & 16.46 & 41.0 & 66.7 & 60.6 & 27.7 & 52.9
& 34.0 \\
\bottomrule
\end{tabular}%
}
\vspace{2pt}
\footnotesize
\textit{FW ppl}: FineWeb-Edu perplexity, \textit{HS}: HellaSwag (acc\_norm),
\textit{PIQA} (acc), \textit{ARC-E}: ARC-Easy (acc),
\textit{ARC-C}: ARC-Challenge (acc\_norm),
\textit{WG}: Winogrande (acc), \textit{LMB}: LAMBADA (acc),
Echo models trained on FineWeb-Edu with $O(1)$ inference memory.
180M-class baselines from~\citep{lahoti2026mamba}.
\end{table*}
The synthetic and MQAR results demonstrate SKA's retrieval advantage in isolation.
To verify that this advantage does not come at the cost of general language understanding, we trained a 180M-parameter Echo model (Mamba-2+SKA+Koopman MLP) from scratch on 10B tokens of FineWeb-Edu and compared it against published baselines on six zero-shot benchmarks (Table~\ref{tab:benchmark}).

Echo-180M achieved the best score among sub-200M models on 5 of 6 benchmarks.
On HellaSwag it scored 44.0, exceeding GPT-2 345M (42.7) at half the parameters and $10{\times}$ fewer tokens.
On LAMBADA, which tests long-range cloze prediction, it scored 39.8, outperforming Mamba-370M ($\sim$36) at half the parameter count.
Echo-50M, trained on only 3B tokens, already matched Pythia-160M and Mamba-130M trained on $30{\times}$ more data, suggesting the architecture is data-efficient.

Taken together, these results show that SKA eliminates the SSM memory cliff on retrieval tasks at both sub-million and 50M scale, while the full Echo architecture matches or exceeds attention-based models on standard language understanding benchmarks with constant-memory inference.
\section{Discussion}
\label{sec:discussion}

\subsection{Design Choices}

\paragraph{Attention as explicit regression.}
Recent theoretical work shows that a softmax self-attention layer trained on linear regression converges to the global optimum of a nonconvex matrix factorization problem, provided the optimizer uses a suitable preconditioner and spectral initialization~\citep{goel2026training}.
SKA makes this regression explicit: rather than learning to approximate the least-squares solution through gradient descent on attention parameters, it computes the ridge-regularized solution in closed form from streaming sufficient statistics.
This eliminates the nonconvex optimization landscape entirely and replaces it with a convex solve whose cost is $O(r^3)$ per step, independent of sequence length.

\paragraph{Independent projections.}
SKA uses its own learned projections from the residual stream rather than Mamba-2's internal features. The SSM recurrence favors smooth, temporally correlated representations, while SKA needs discriminative inner products for content addressing. Sharing features degrades both.

\paragraph{Whitening and spectral normalization.}
Whitening via $A_w = L^{-1} M L^{-\top}$ makes eigenvalues interpretable as persistence ($|\lambda| \approx 1$) versus transience ($|\lambda| \ll 1$).
This mirrors the role of preconditioning in the convergence analysis of~\citet{goel2026training}, where a structure-aware preconditioner is required to avoid spurious stationary points in attention's implicit optimization; in SKA the Cholesky whitening serves the analogous function of ensuring the operator's spectrum is well-conditioned for the power filter.
Spectral normalization bounds the operator to prevent gradient spikes through the power filter, with a learned scalar $\gamma \in [1.0, 1.5]$ restoring variance.
The power filter at $K{=}2$ retains 81\% energy at $|\lambda|{=}0.9$ while suppressing $|\lambda|{=}0.3$ to 9\%.

\paragraph{Additive injection.}
A learned sigmoid gate failed due to sparse supervision (2--4 answer tokens per sequence).
Direct additive injection $h \leftarrow h + \eta \cdot h_{\text{ska}}$ with $\eta \in \{1.5, 2.5\}$ ensures sufficient gradient from the start of training.

\paragraph{Numerical stability.}
All covariance accumulation and linear algebra is performed in FP32.
Sequence-max normalization and ridge regularization ($\varepsilon{=}10^{-3}$) prevent ill-conditioning.

\subsection{Computational Cost}
SKA's inference state is $2r^2 + Pr + r$ floats per head, fixed regardless of sequence length. For the 180M model ($r{=}56$, $P{=}64$, $H{=}12$), total SKA state across 2 layers is approximately 77\,KB.
An equivalent attention layer's KV cache at 131K tokens with $d{=}768$ would require 384\,MB in FP16 per layer.
Flash attention~\citep{dao2023flashattention} reduces the constant factor but memory remains $O(Td)$: the full KV cache must still be stored and scanned at each decoding step.

The current bottleneck is the $O(r^3)$ Cholesky factorization and spectral normalization, which require FP32 and do not fully exploit GPU parallelism at small matrix sizes. Our
Triton kernel fuses the post-Cholesky matmul chain into a single SRAM-resident pass, but the factorization itself remains in PyTorch.
Fusing these steps into custom kernels, potentially using approximate factorizations to relax the FP32 requirement, is active work and we expect it to substantially reduce wall-clock overhead.

\subsection{Limitations}
\paragraph{Scale.}
Our largest model is 180M parameters trained on 10B tokens on a single NVIDIA B200, reflecting compute budget constraints rather than architectural limitations.
Validating Echo at billion-parameter scale remains future work.

\paragraph{Synthetic-to-natural gap.}
The synthetic benchmarks have clean fact/distractor/query structure that natural language lacks.
The 180M language modeling results suggest no degradation in general understanding, but targeted retrieval evaluation on natural-language benchmarks (RULER, BABILong) is needed.
\section{Conclusion}
State-space models compress sequence history into a fixed-size recurrent state, and we showed that this compression produces a memory cliff: on MQAR, a pure Mamba-2 scores ${\sim}3\%$ regardless of model scale, and on five transfer retrieval tasks, accuracy degrades sharply as sequence length grows.
This failure is structural, not a capacity limitation.
Echo addresses this through Spectral Koopman Attention (SKA), which recasts content-addressed retrieval as kernel ridge regression over streaming sufficient statistics of fixed size $O(r^2 + rP)$.
The whitened Koopman operator identifies persistent bindings via its eigenspectrum, and a power-iterated filter retrieves them at query time with no KV cache.
At the 50M parameter scale, Echo achieved perfect MQAR retrieval on every configuration tested.
At sub-million scale, Echo outperformed both pure SSM and SSM+Attention baselines on all five transfer benchmarks while using the fewest parameters (982K vs 998K and 1.11M) and maintaining constant inference memory.
Length generalization experiments showed Echo retaining 65\% accuracy at $64{\times}$ the training sequence length, where both baselines collapsed to ${\leq}5\%$.
These retrieval gains did not come at the cost of general language understanding. Echo-180M, trained from scratch on 10B tokens, matched or exceeded GPT-2 345M on 5 of 6 zero-shot benchmarks despite $2{\times}$ fewer parameters and $10{\times}$ fewer training tokens, suggesting the architecture is data-efficient as well as memory-efficient.
Several limitations remain. Our experiments reflect compute budget constraints (a single B200 GPU), and validating Echo at billion-parameter scale with trillion-token budgets is needed.
The synthetic benchmarks have clean fact/distractor/query structure that natural language lacks; targeted evaluation on retrieval benchmarks such as RULER and BABILong will determine whether Echo's advantage transfers to realistic agentic workloads.
Fusing the Cholesky factorization and spectral normalization into custom kernels to reduce wall-clock overhead is active work in progress.
The core result is that a principled dynamical-systems mechanism can close the retrieval gap with attention without inheriting its linear memory cost. This points toward architectures where recurrence handles local context while lightweight operator modules handle the sparse, long-range lookups that agentic workflows demand, all within a fixed memory envelope.
\bibliographystyle{ACM-Reference-Format}
\bibliography{references}

\appendix
\providecommand{\R}{\mathbb{R}}
\providecommand{\C}{\mathbb{C}}
\providecommand{\E}{\mathbb{E}}
\providecommand{\T}{^{\mathsf T}}
\providecommand{\Id}{\mathrm{I}}
\providecommand{\diag}{\mathrm{diag}}
\providecommand{\tr}{\mathrm{tr}}
\providecommand{\rank}{\mathrm{rank}}
\providecommand{\cO}{\mathcal{O}}

\renewcommand{\thesection}{\texorpdfstring{$\mathbb{\Alph{section}}$}{\Alph{section}}}

\section{Koopman Operator Estimation}
\label{app:koopman}

This appendix provides complete derivations for the
mathematical claims made in the main text.
We proceed from the basic least-squares estimator through
whitening, spectral filtering, and perturbation analysis,
giving explicit proofs at each step.

\subsection{Least-Squares Koopman Estimator}
\label{app:koopman_ls}

\begin{proposition}[Koopman least-squares estimator]
\label{prop:koopman_ls}
Let $\{z_t\}_{t=1}^{T}$ be a sequence of feature vectors
$z_t \in \R^r$.
Assume an approximate linear dynamical model
$z_{t+1} \approx A z_t$.
The least-squares estimator
\begin{equation}
    A^* = \arg\min_{A \in \R^{r \times r}}
    \sum_{t=1}^{T-1} \|z_{t+1} - A z_t\|^2
\end{equation}
is given by $A^* = M G^{-1}$, where
\begin{equation}
    G = \sum_{t=1}^{T-1} z_t z_t\T, \qquad
    M = \sum_{t=1}^{T-1} z_{t+1} z_t\T,
\end{equation}
provided $G$ is invertible.
\end{proposition}

\begin{proof}
Expanding the objective:
\begin{align}
    \cO(A) &= \sum_{t=1}^{T-1}
    (z_{t+1} - A z_t)\T (z_{t+1} - A z_t) \\
    &= \sum_t z_{t+1}\T z_{t+1}
    - 2 \tr\!\Big(A \sum_t z_t z_{t+1}\T\Big)
    + \tr\!\Big(A \sum_t z_t z_t\T A\T\Big).
\end{align}
Setting $\partial \cO / \partial A = 0$:
\begin{equation}
    -2 M\T + 2 A G = 0
    \quad \Longrightarrow \quad
    A = M\T\! G^{-1}.
\end{equation}
Since $M = \sum_t z_{t+1} z_t\T$, we have
$M\T = \sum_t z_t z_{t+1}\T$, so $A^* = M G^{-1}$ after
transposing both sides of the normal equation
(equivalently, taking the derivative with respect to $A$
in the standard matrix calculus convention where
$\partial \tr(A X) / \partial A = X\T$).
\end{proof}

\begin{remark}
The lagged structure is essential: $M$ uses $(z_{t+1}, z_t)$
pairs, not $(z_t, z_t)$.
If $M$ were replaced by $G$, the estimator would yield
$A = I$ regardless of the data, which is uninformative.
\end{remark}

\subsection{Ridge Regularization}
\label{app:ridge}

When $T - 1 < r$ or when features are nearly collinear,
$G$ is singular or ill-conditioned.
We replace $G$ with the regularized Gram matrix
\begin{equation}
    \tilde{G} = G + \varepsilon I, \qquad \varepsilon > 0,
\end{equation}
yielding the ridge estimator
$A_\varepsilon = M \tilde{G}^{-1}$.

\begin{proposition}[Ridge as shrinkage]
\label{prop:ridge_shrinkage}
Let $G = \sum_t z_t z_t\T$ have eigendecomposition
$G = V \Sigma V\T$ with eigenvalues
$\sigma_1 \geq \cdots \geq \sigma_r \geq 0$.
Then the ridge estimator satisfies
\begin{equation}
    A_\varepsilon = M V
    \diag\!\bigg(\frac{1}{\sigma_1 + \varepsilon}, \ldots,
    \frac{1}{\sigma_r + \varepsilon}\bigg) V\T.
\end{equation}
In particular, directions with small data support
($\sigma_j \ll \varepsilon$) are shrunk toward zero:
the corresponding component of $A_\varepsilon$ scales as
$\sigma_j / (\sigma_j + \varepsilon) \to 0$.
\end{proposition}

\begin{proof}
Since $\tilde{G} = V(\Sigma + \varepsilon I)V\T$, we have
$\tilde{G}^{-1} = V(\Sigma + \varepsilon I)^{-1} V\T$.
Substituting into $A_\varepsilon = M \tilde{G}^{-1}$ gives
the result.
The shrinkage factor for the $j$-th eigendirection of $G$ is
$\sigma_j / (\sigma_j + \varepsilon)$, which vanishes as
$\sigma_j \to 0$.
\end{proof}

\subsection{Whitening and Spectral Equivalence}
\label{app:whitening}

\begin{proposition}[Whitened operator]
\label{prop:whitening}
Let $\tilde{G} = L L\T$ be the Cholesky factorization of the
regularized Gram matrix.
Define the whitened Koopman operator
\begin{equation}
    A_w = L^{-1} M L^{-\top}.
\end{equation}
Then $A_w$ and $A_\varepsilon = M \tilde{G}^{-1}$ have the
same eigenvalues.
\end{proposition}

\begin{proof}
We show that $A_w$ and $A_\varepsilon$ are related by a
similarity transform.
Write
\begin{equation}
    A_\varepsilon = M \tilde{G}^{-1} = M (L L\T)^{-1}
    = M L^{-\top} L^{-1}.
\end{equation}
Now observe:
\begin{align}
    L^{-1} A_\varepsilon L
    &= L^{-1} (M L^{-\top} L^{-1}) L \\
    &= L^{-1} M L^{-\top} \\
    &= A_w.
\end{align}
Since $A_w = L^{-1} A_\varepsilon L$, the two matrices are
similar and therefore share the same eigenvalues.
\end{proof}

\begin{remark}
Although $A_w$ and $A_\varepsilon$ share eigenvalues, they
differ in conditioning.
The na\"ive product $\tilde{G}^{-1} M$ (note the reversed
order) is \emph{not} similar to $A_\varepsilon$ in general,
and can have a different eigenspectrum.
The whitened form $L^{-1} M L^{-\top}$ is the correct
implementation.
\end{remark}

\subsection{Positive Definiteness of the Regularized Gram}
\label{app:pd}

\begin{lemma}
\label{lem:pd}
For any $\varepsilon > 0$, the regularized Gram matrix
$\tilde{G} = G + \varepsilon I$ is positive definite and
satisfies $\lambda_{\min}(\tilde{G}) \geq \varepsilon$.
\end{lemma}

\begin{proof}
$G = \sum_t z_t z_t\T$ is positive semidefinite since for any
$u \in \R^r$,
$u\T G u = \sum_t (z_t\T u)^2 \geq 0$.
Therefore
$u\T \tilde{G} u = u\T G u + \varepsilon \|u\|^2
\geq \varepsilon \|u\|^2 > 0$
for all $u \neq 0$.
The minimum eigenvalue bound follows from
$\lambda_{\min}(\tilde{G}) = \lambda_{\min}(G) + \varepsilon
\geq \varepsilon$.
\end{proof}

\section{Clebsch--Gordan Invariant Lift}
\label{app:cg}

\subsection{Rotation Group Action on SSM Keys}
\label{app:cg_action}

The Mamba-3 rotary mechanism applies a position-dependent
phase to each frequency component of the key vector.
In complex notation, the $j$-th component of the rotated key
at position $t$ is
\begin{equation}
    k_{t,j} = \tilde{k}_{t,j} \, e^{i\theta_{t,j}},
\end{equation}
where $\tilde{k}_{t,j} \in \C$ is the content-dependent base
key and $\theta_{t,j} = \sum_{s=1}^{t} \delta_s \omega_j$ is
the accumulated phase.
In real coordinates, this corresponds to an $\mathrm{SO}(2)$
rotation of each even-odd pair by angle $\theta_{t,j}$.
The full rotation acts as $\mathrm{SO}(2)^{N/2}$ on
$\R^N$.

\subsection{Degree-2 Tensor Product Decomposition}
\label{app:cg_decomp}

\begin{proposition}[CG decomposition for $\mathrm{SO}(2)$]
\label{prop:cg}
Let $z \in \C$ transform under $\mathrm{SO}(2)$ as
$z \mapsto z e^{i\theta}$.
The degree-2 tensor products of $z$ decompose into
irreducible representations as follows:
\begin{enumerate}[leftmargin=1.25em]
    \item The invariant (trivial representation):
    $z \bar{z} = |z|^2$, satisfying
    $|z e^{i\theta}|^2 = |z|^2$ for all $\theta$.
    \item The second harmonic (rotation by $2\theta$):
    $z^2$, satisfying
    $(z e^{i\theta})^2 = z^2 e^{i 2\theta}$.
\end{enumerate}
\end{proposition}

\begin{proof}
For the invariant:
$|z e^{i\theta}|^2 = (z e^{i\theta})(\overline{z e^{i\theta}})
= z \bar{z} e^{i\theta} e^{-i\theta} = |z|^2$.

For the harmonic:
$(z e^{i\theta})^2 = z^2 e^{i 2\theta}$, which transforms
under the $\mathrm{SO}(2)$ representation of angular
frequency 2.
\end{proof}

\subsection{Invariance of the Feature Map}
\label{app:cg_invariance}

\begin{theorem}[Phase cancellation in the invariant lift]
\label{thm:phase_cancel}
Define the invariant feature map
$\psi_{\mathrm{inv}}: \C^{N/2} \to \R^{N/2}$ by
\begin{equation}
    [\psi_{\mathrm{inv}}(k)]_j = |k_j|^2
    = (k_j^{\Re})^2 + (k_j^{\Im})^2.
\end{equation}
If $k_t = \tilde{k}_t \odot e^{i\theta_t}$ componentwise,
then
\begin{equation}
    \psi_{\mathrm{inv}}(k_t) = \psi_{\mathrm{inv}}(\tilde{k}_t)
\end{equation}
for all $\theta_t$.
In particular, the Gram matrix
$G = \sum_t \psi_{\mathrm{inv}}(k_t)
\psi_{\mathrm{inv}}(k_t)\T$
is independent of the position encodings
$\{\theta_t\}$.
\end{theorem}

\begin{proof}
By Proposition~\ref{prop:cg} applied componentwise,
$|k_{t,j}|^2 = |\tilde{k}_{t,j}|^2$ for each $j$.
Therefore $\psi_{\mathrm{inv}}(k_t) =
\psi_{\mathrm{inv}}(\tilde{k}_t)$, and any function of
$\{\psi_{\mathrm{inv}}(k_t)\}_{t=1}^T$ is independent of
$\{\theta_t\}$.
\end{proof}

\begin{corollary}[Content-addressed Gram matrix]
\label{cor:content_gram}
Under the invariant lift, the inner product
$\langle \psi_{\mathrm{inv}}(k_s),
\psi_{\mathrm{inv}}(k_t) \rangle$
depends only on the content vectors $\tilde{k}_s, \tilde{k}_t$
and not on positions $s, t$.
Thus $G$ functions as a content-addressed index.
\end{corollary}

\subsection{Harmonic Features and Equivariance}
\label{app:cg_harmonic}

The harmonic channel retains relative phase information.
In real coordinates:
\begin{align}
    \Re(k_j^2) &= (k_j^{\Re})^2 - (k_j^{\Im})^2, \\
    \Im(k_j^2) &= 2 k_j^{\Re} k_j^{\Im}.
\end{align}
Under rotation by $\theta$, these transform as
\begin{equation}
    \begin{pmatrix}
    \Re((k_j e^{i\theta})^2) \\
    \Im((k_j e^{i\theta})^2)
    \end{pmatrix}
    =
    \begin{pmatrix}
    \cos 2\theta & -\sin 2\theta \\
    \sin 2\theta & \cos 2\theta
    \end{pmatrix}
    \begin{pmatrix}
    \Re(k_j^2) \\
    \Im(k_j^2)
    \end{pmatrix}.
\end{equation}
The harmonic features are not invariant but are
\emph{equivariant}: they transform predictably under the
group action.
The full feature map $\psi(k) = [\psi_{\mathrm{inv}}(k);
\Re(k^2); \Im(k^2)] \in \R^{3N/2}$ combines invariant and
equivariant channels, allowing the learned projection $\Phi$
to select the combination best suited to the task.

\section{Perturbation Analysis}
\label{app:perturbation}

\subsection{Operator Perturbation Bound}
\label{app:op_perturb}

\begin{theorem}[Perturbation of the Koopman estimator]
\label{thm:op_perturb}
Let $G, M$ be the true covariance matrices and let
$\hat{G} = G + E_G$, $\hat{M} = M + E_M$ be empirical
estimates with perturbations $E_G, E_M$.
Define $\tilde{G} = G + \varepsilon I$ and
$\hat{\tilde{G}} = \hat{G} + \varepsilon I = \tilde{G} + E_G$.
Then for the Koopman estimators
$A = M \tilde{G}^{-1}$ and
$\hat{A} = \hat{M} \hat{\tilde{G}}^{-1}$:
\begin{equation}
    \|\hat{A} - A\| \leq
    \|E_M\| \cdot \|\hat{\tilde{G}}^{-1}\|
    + \|A\| \cdot \|\tilde{G}^{-1}\| \cdot \|E_G\|
    \cdot \|\hat{\tilde{G}}^{-1}\| \cdot \|\tilde{G}\|.
\end{equation}
In particular, the perturbation is controlled by
$\|\tilde{G}^{-1}\| = 1 / \lambda_{\min}(\tilde{G})$.
\end{theorem}

\begin{proof}
Write
\begin{align}
    \hat{A} - A
    &= \hat{M} \hat{\tilde{G}}^{-1} - M \tilde{G}^{-1} \\
    &= (M + E_M) \hat{\tilde{G}}^{-1} - M \tilde{G}^{-1} \\
    &= E_M \hat{\tilde{G}}^{-1}
    + M \big(\hat{\tilde{G}}^{-1} - \tilde{G}^{-1}\big).
\end{align}
For the second term, apply the resolvent identity.
Since $\hat{\tilde{G}} = \tilde{G} + E_G$:
\begin{equation}
    \hat{\tilde{G}}^{-1} - \tilde{G}^{-1}
    = -\tilde{G}^{-1} E_G \hat{\tilde{G}}^{-1}.
\end{equation}
To verify: multiply both sides on the left by $\tilde{G}$
and on the right by $\hat{\tilde{G}}$:
\begin{align}
    \tilde{G}(\hat{\tilde{G}}^{-1} - \tilde{G}^{-1})
    \hat{\tilde{G}}
    &= \tilde{G} \hat{\tilde{G}}^{-1} \hat{\tilde{G}}
    - \hat{\tilde{G}} \\
    &= \tilde{G} - \hat{\tilde{G}} \\
    &= -E_G,
\end{align}
confirming the identity.
Substituting:
\begin{equation}
    \hat{A} - A = E_M \hat{\tilde{G}}^{-1}
    - M \tilde{G}^{-1} E_G \hat{\tilde{G}}^{-1}.
\end{equation}
Taking norms and using $M = A \tilde{G}$:
\begin{align}
    \|\hat{A} - A\|
    &\leq \|E_M\| \cdot \|\hat{\tilde{G}}^{-1}\|
    + \|A\| \cdot \|\tilde{G}\| \cdot
    \|\tilde{G}^{-1}\| \cdot \|E_G\|
    \cdot \|\hat{\tilde{G}}^{-1}\|. \qedhere
\end{align}
\end{proof}

\begin{corollary}[Role of $\lambda_{\min}(\tilde{G})$]
\label{cor:lmin}
Since $\|\tilde{G}^{-1}\|_2 = 1/\lambda_{\min}(\tilde{G})$,
increasing the minimum eigenvalue of the regularized Gram
matrix directly reduces the worst-case amplification of
covariance noise into the operator estimate.
This motivates logging $\lambda_{\min}(\tilde{G})$ as a
stability diagnostic.
\end{corollary}

\subsection{Eigenvalue Perturbation}
\label{app:eig_perturb}

\begin{theorem}[Bauer--Fike]
\label{thm:bauer_fike}
Let $A \in \R^{r \times r}$ be diagonalizable with
$A = V \Lambda V^{-1}$.
For any eigenvalue $\hat{\lambda}$ of $\hat{A} = A + E$,
there exists an eigenvalue $\lambda$ of $A$ such that
\begin{equation}
    |\hat{\lambda} - \lambda|
    \leq \kappa(V) \cdot \|E\|,
\end{equation}
where $\kappa(V) = \|V\| \cdot \|V^{-1}\|$ is the condition
number of the eigenvector matrix.
\end{theorem}

\begin{proof}
This is the classical Bauer--Fike theorem.
If $\hat{\lambda}$ is not an eigenvalue of $A$, then
$(\hat{\lambda} I - A)$ is invertible and
$\hat{A} v = \hat{\lambda} v$ implies
$v = (\hat{\lambda} I - A)^{-1} E v$.
Taking norms:
$1 \leq \|(\hat{\lambda} I - A)^{-1}\| \cdot \|E\|$.
Writing $A = V \Lambda V^{-1}$:
\begin{equation}
    (\hat{\lambda} I - A)^{-1}
    = V (\hat{\lambda} I - \Lambda)^{-1} V^{-1},
\end{equation}
so $\|(\hat{\lambda} I - A)^{-1}\| \leq \kappa(V)
/ \min_j |\hat{\lambda} - \lambda_j|$.
Rearranging gives
$\min_j |\hat{\lambda} - \lambda_j|
\leq \kappa(V) \|E\|$.
\end{proof}

\begin{corollary}[Combined bound]
\label{cor:combined}
Combining Theorem~\ref{thm:op_perturb} and
Theorem~\ref{thm:bauer_fike}, the eigenvalue perturbation
of the Koopman operator due to finite-sample estimation
errors is bounded by
\begin{equation}
    \min_j |\hat{\lambda}_j - \lambda_j|
    \leq \kappa(V_{A_w}) \bigg(
    \frac{\|E_M\|}{\lambda_{\min}(\hat{\tilde{G}})}
    + \frac{\|A\| \cdot \|\tilde{G}\| \cdot \|E_G\|}
    {\lambda_{\min}(\tilde{G}) \cdot
    \lambda_{\min}(\hat{\tilde{G}})}
    \bigg).
\end{equation}
This makes explicit that eigenvalue stability requires both
a well-conditioned Gram matrix (large $\lambda_{\min}$)
and a well-conditioned eigenvector matrix of the operator
(small $\kappa(V_{A_w})$).
\end{corollary}

\subsection{Why Phase Nonstationarity Harms the Gram Matrix}
\label{app:phase_harm}

\begin{proposition}[Phase-induced conditioning loss]
\label{prop:phase_cond}
Consider scalar complex features
$k_t = \tilde{k}_t e^{i\theta_t}$ where $\tilde{k}_t$ and
$\theta_t$ are independent, and $\theta_t$ is uniformly
distributed on $[0, 2\pi)$.
Let $G_{\mathrm{raw}} = \sum_t k_t \bar{k}_t = \sum_t
|\tilde{k}_t|^2$ (which is real and positive) denote the
diagonal Gram entry.
For the cross-covariance:
\begin{equation}
    M_{\mathrm{raw}} = \sum_t k_{t+1} \bar{k}_t
    = \sum_t \tilde{k}_{t+1} \bar{\tilde{k}}_t \,
    e^{i(\theta_{t+1} - \theta_t)}.
\end{equation}
If $\theta_{t+1} - \theta_t$ has high variance and is
weakly correlated with the content factors, the oscillatory
terms $e^{i(\theta_{t+1} - \theta_t)}$ cause partial
cancellation in $M_{\mathrm{raw}}$, reducing
$|M_{\mathrm{raw}}|$ relative to $G_{\mathrm{raw}}$.
The resulting operator $A = M_{\mathrm{raw}} /
G_{\mathrm{raw}}$ is attenuated toward zero, obscuring the
true dynamics.
\end{proposition}

\begin{proof}
In the extreme case where $\theta_{t+1} - \theta_t$ is
i.i.d.\ uniform on $[0, 2\pi)$:
\begin{equation}
    \E[e^{i(\theta_{t+1} - \theta_t)}] = 0.
\end{equation}
Therefore $\E[M_{\mathrm{raw}}] = 0$ regardless of the
content dynamics, while
$\E[G_{\mathrm{raw}}] = \sum_t \E[|\tilde{k}_t|^2] > 0$.
The operator estimate $A = M G^{-1}$ converges to zero
rather than the true dynamics.

The invariant lift $\psi(k_t) = |k_t|^2 = |\tilde{k}_t|^2$
removes the oscillatory factor entirely, so that
$M_{\mathrm{inv}} = \sum_t |\tilde{k}_{t+1}|^2
|\tilde{k}_t|^2$ reflects the true content dynamics
without cancellation.
\end{proof}

\section{Spectral Filtering}
\label{app:spectral}

\subsection{Power Filter Analysis}
\label{app:power_filter}

\begin{proposition}[Mode separation under the power filter]
\label{prop:power_sep}
Let $A_w$ have eigendecomposition
$A_w = V \Lambda V^{-1}$ with eigenvalues
$\lambda_1, \ldots, \lambda_r$.
The power filter applies $A_w^K$ to a query vector $w_q$,
which acts as $V \Lambda^K V^{-1} w_q$.
Each eigenvalue is raised to the $K$-th power:
$\lambda_j \mapsto \lambda_j^K$.

For $K = 2$:
\begin{center}
\begin{tabular}{@{}cc@{}}
\toprule
$|\lambda|$ & $|\lambda|^2$ \\
\midrule
$0.95$ & $0.902$ \\
$0.90$ & $0.810$ \\
$0.70$ & $0.490$ \\
$0.50$ & $0.250$ \\
$0.30$ & $0.090$ \\
\bottomrule
\end{tabular}
\end{center}

For $K = 4$:
\begin{center}
\begin{tabular}{@{}cc@{}}
\toprule
$|\lambda|$ & $|\lambda|^4$ \\
\midrule
$0.95$ & $0.815$ \\
$0.90$ & $0.656$ \\
$0.70$ & $0.240$ \\
$0.50$ & $0.063$ \\
$0.30$ & $0.008$ \\
\bottomrule
\end{tabular}
\end{center}
\end{proposition}

\begin{proof}
If $A_w = V \Lambda V^{-1}$, then
$A_w^K = V \Lambda^K V^{-1}$ by induction on $K$.
The entries of $\Lambda^K$ are $\lambda_j^K$, giving the
tabulated values.
\end{proof}

\begin{remark}
At $K = 2$, modes with $|\lambda| \geq 0.7$ retain at least
49\% of their energy, providing a wide band of usable
persistent modes.
At $K = 4$, only modes with $|\lambda| \geq 0.85$ retain
more than 50\%, which can be too aggressive when the operator
has many modes in the $0.6$--$0.8$ range that carry useful
partial information.
\end{remark}

\subsection{Contractivity After Spectral Normalization}
\label{app:spectral_norm}

\begin{proposition}[Power filter contractivity]
\label{prop:contractivity}
Let $\hat{A}_w = A_w / \max(1, \sigma_{\max}(A_w))$, where
$\sigma_{\max}$ denotes the spectral norm (largest singular
value).
Then $\sigma_{\max}(\hat{A}_w) \leq 1$, and for any $K \geq 1$:
\begin{equation}
    \|\hat{A}_w^K\| \leq 1.
\end{equation}
\end{proposition}

\begin{proof}
By construction,
$\sigma_{\max}(\hat{A}_w) = \sigma_{\max}(A_w) /
\max(1, \sigma_{\max}(A_w)) \leq 1$.
For any matrix with spectral norm at most 1,
$\|A^K\| \leq \|A\|^K \leq 1$ by submultiplicativity.
\end{proof}

\begin{remark}
This guarantees that the power filter cannot amplify any
component of the query vector, preventing the gradient
spikes observed when the unnormalized operator has
$\sigma_{\max} \gg 1$.
\end{remark}

\subsection{Variance Restoration via SSN}
\label{app:ssn}

Spectral normalization divides all eigenvalues by
$\sigma_{\max}$, which can push persistent modes
(those near $|\lambda| = 1$) below the effective filter
threshold.

\begin{proposition}[Controlled restoration]
\label{prop:ssn}
Let $\hat{A}_w$ be the spectrally normalized operator with
$\sigma_{\max}(\hat{A}_w) \leq 1$.
Applying a learned scalar $\gamma \in [1, \gamma_{\max}]$
yields $A_w' = \gamma \hat{A}_w$ with
\begin{equation}
    \sigma_{\max}(A_w') = \gamma \cdot
    \sigma_{\max}(\hat{A}_w) \leq \gamma_{\max}.
\end{equation}
For $\gamma_{\max} = 1.5$ and $K = 2$:
\begin{equation}
    \|{A_w'}^K\| \leq \gamma_{\max}^K = 2.25,
\end{equation}
which is bounded and does not grow with model dimensions.
The power filter remains well-behaved while eigenvalues
in the range $[1/\gamma_{\max}, 1]$ can be restored to
magnitudes above the filter threshold.
\end{proposition}

\begin{proof}
$\sigma_{\max}(A_w') = |\gamma| \cdot
\sigma_{\max}(\hat{A}_w) \leq \gamma_{\max} \cdot 1
= \gamma_{\max}$.
The power filter norm bound follows from
submultiplicativity.
\end{proof}

\section{Sequence-Max Normalization}
\label{app:seqmax}

\begin{proposition}[Gram matrix bound under sequence-max normalization]
\label{prop:seqmax}
Let $\{z_t\}_{t=1}^T$ be feature vectors and define
$\bar{z}_t = z_t / \max_s \|z_s\|$.
Then $\|\bar{z}_t\| \leq 1$ for all $t$, and the Gram matrix
$\bar{G} = \sum_t \bar{z}_t \bar{z}_t\T$ satisfies
\begin{equation}
    \|\bar{G}\|_F \leq T, \qquad
    \tr(\bar{G}) \leq T, \qquad
    \lambda_{\max}(\bar{G}) \leq T.
\end{equation}
\end{proposition}

\begin{proof}
By construction $\|\bar{z}_t\| \leq 1$, so each rank-1 term
$\bar{z}_t \bar{z}_t\T$ has Frobenius norm
$\|\bar{z}_t\|^2 \leq 1$ and trace $\|\bar{z}_t\|^2 \leq 1$.
Summing over $T$ terms gives the bounds.
\end{proof}

\begin{remark}
Unlike per-token $\ell_2$-normalization (which sets
$\|z_t\| = 1$ for all $t$, inflating low-norm noise tokens
to unit norm), sequence-max normalization preserves the
relative norm structure within a sequence: high-norm
fact-table tokens remain dominant while low-norm distractor
tokens are not amplified.
\end{remark}

\section{SKA Retrieval Pipeline}
\label{app:pipeline}

For reference, we state the complete SKA computation as a
single pipeline, consolidating the results above.

\begin{enumerate}
    \item \textbf{Feature extraction.}
    Given the post-Mamba residual $h_t \in \R^d$, compute
    per-head features:
    \begin{equation}
        z_t = W_K h_t \in \R^r, \quad
        z_t^{(q)} = W_Q h_t \in \R^r, \quad
        v_t = W_V h_t \in \R^P.
    \end{equation}

    \item \textbf{Sequence-max normalization.}
    Let $\nu = \max_t \|z_t\|$.
    Replace $z_t \leftarrow z_t / \nu$ and
    $z_t^{(q)} \leftarrow z_t^{(q)} / \nu$.

    \item \textbf{Masked accumulation.}
    Using the action mask $m_t \in \{0, 1\}$:
    \begin{equation}
        G = \sum_{t: m_t = 1} z_t z_t\T, \qquad
        M = \sum_{\substack{t: m_t = 1 \\ m_{t+1} = 1}}
        z_{t+1} z_t\T, \qquad
        C = \sum_{t: m_t = 1} v_t z_t\T.
    \end{equation}

    \item \textbf{Regularization and Cholesky.}
    $\tilde{G} = G + \varepsilon I$, \quad
    $\tilde{G} = L L\T$ \quad
    (Lemma~\ref{lem:pd} guarantees existence).

    \item \textbf{Whitened Koopman operator.}
    Solve $L Y = M$ and $L Z = Y\T$ by forward substitution,
    then $A_w = Z\T$
    (Proposition~\ref{prop:whitening}).

    \item \textbf{Spectral normalization.}
    Estimate $\sigma_{\max}(A_w)$ via power iteration
    (6 steps). Set
    $\hat{A}_w = A_w / \max(1, \sigma_{\max})$,
    then $A_w' = \gamma \hat{A}_w$ with learned
    $\gamma \in [1.0, 1.5]$
    (Proposition~\ref{prop:contractivity},
    Proposition~\ref{prop:ssn}).

    \item \textbf{Value map.}
    Solve $\tilde{G} X = C\T$, set $B_v = X\T \in \R^{P \times r}$.

    \item \textbf{Query whitening.}
    Solve $L w_q = z^{(q)}$ by forward substitution.

    \item \textbf{Power filter.}
    Compute $w_f = {A_w'}^K w_q$
    (Proposition~\ref{prop:power_sep}).

    \item \textbf{Un-whiten and retrieve.}
    $z_f = L w_f$, \quad
    $\hat{y} = \eta \cdot B_v \, z_f \in \R^P$.
\end{enumerate}

All operations in steps 3--9 are performed in FP32.
The output $\hat{y}$ is cast back to the training precision
and added to the residual stream.
\section{Experimental Setup}
\label{app:experimental-setup}
\subsection{Model Architectures}

We evaluate two model families across all experiments:

\paragraph{Mamba-3 (Baseline).}
A pure state-space language model built from stacked Mamba-3 blocks.
Each block consists of: (i) a LayerNorm pre-normalization, (ii) a fused input
projection that produces keys~$\mathbf{b}$, queries~$\mathbf{c}$,
values~$\mathbf{v}$, step size~$\Delta$, decay~$a$, and interpolation
weight~$\lambda$ from a single linear map, (iii) a depth-wise 1-D convolution
(kernel size~4) with SiLU activation on the value branch, (iv) a complex-valued
selective state-space recurrence with Rotary Position Embeddings (RoPE) on
keys and queries, computed via a chunked Structured State-Space Duality (SSD)
scan in FP32 with chunk size~$C{=}64$, and (v)~a linear output projection.
Each Mamba-3 block is followed by a SwiGLU MLP with expansion
ratio~$\tfrac{8}{3}$ (inner dimension rounded up to the nearest multiple of~64).

\paragraph{Mamba-3 + CG + SKA.}
The baseline augmented with Spectral Koopman Attention (SKA) modules inserted at the final two layers.
SKA provides a complementary non-recurrent retrieval pathway via a learned Koopman operator.
Key design choices:
\begin{itemize}
  \item \textbf{Independent projections}: SKA uses its own
        key/query/value projections from the residual stream
        (orthogonally initialized), decoupled from Mamba's rotary features.
  \item \textbf{Whitened Koopman operator}:
        $A_w = L^{-1} M L^{-\top}$ where $L$ is the Cholesky factor of
        the ridge-regularized Gram matrix $\tilde{G} = G + \epsilon I$
        ($\epsilon{=}10^{-3}$).
  \item \textbf{Scaled Spectral Normalization (SSN)}:
        $A_w$ is spectrally normalized via 6-step power iteration,
        then rescaled by a learnable $\gamma \in [1.0, 1.5]$.
  \item \textbf{Power spectral filter}: $A_w^K$ with $K{=}2$
        suppresses transient modes while preserving persistent ones.
  \item \textbf{Sequence-max normalization} on lifted features.
  \item \textbf{Ungated additive injection}: the SKA output is
        projected and added directly to the residual stream
        (no sigmoid gate, which cannot learn from ${\leq}2$ answer tokens per example).
  \item \textbf{Action masking}: only high-signal tokens contribute
        to the Gram and cross-covariance matrices.
\end{itemize}

\subsection{Hyperparameters}

Table~\ref{tab:arch-hparams} summarizes the architectural hyperparameters used in each experiment.  All configurations share a state dimension~$N{=}16$, chunk size~$C{=}64$, SKA rank~$r{=}24$, ridge~$\epsilon{=}10^{-3}$, and power filter order~$K{=}2$.

\begin{table}[h]
\centering
\caption{Architectural hyperparameters by experiment.
         $d$: model dimension; $H$: heads; $L$: layers; $|V|$: vocab size;
         $d_{\text{head}} = d / H$; $d_{\text{ff}}$: SwiGLU inner dimension
         ($\lceil d \cdot 8/3 \rceil_{64}$).}
\label{tab:arch-hparams}
\small
\begin{tabular}{lcccccc}
\toprule
\textbf{Experiment} & $d$ & $H$ & $L$ & $|V|$ & $d_{\text{head}}$ & $d_{\text{ff}}$ \\
\midrule
Tool-Calling     & 96 & 3 & 3 & 128 & 32 & 256 \\
CoT Retrieval    & 96 & 4 & 4 & 128 & 24 & 256 \\
\bottomrule
\end{tabular}
\end{table}

For SKA-augmented models, the module is attached at layers~$L{-}2$
and~$L{-}1$ (the final two layers).  We test two SKA scaling
coefficients:  $\eta{=}1.5$ (standard) and $\eta{=}2.5$ (strong)
to demonstrate a dose--response relationship.

\subsection{Training Configuration}

All models are trained with the following protocol:

\begin{table}[h]
\centering
\caption{Training hyperparameters.}
\label{tab:train-hparams}
\small
\begin{tabular}{ll}
\toprule
\textbf{Parameter} & \textbf{Value} \\
\midrule
Optimizer          & AdamW ($\beta_1{=}0.9$, $\beta_2{=}0.95$) \\
Weight decay       & $0.01$ \\
Peak learning rate & $3 \times 10^{-4}$ \\
LR schedule        & Cosine decay with linear warmup \\
Steps              & 2000-6000 \\
Gradient clipping  & Max-norm $= 1.0$ \\
Batch size         & 16 \\
Precision          & Mixed FP16 (AMP) when GPU available \\
\bottomrule
\end{tabular}
\end{table}

\subsection{Evaluation Tasks}

\paragraph{Experiment~1: Tool-Calling Retrieval.}
Sequences encode a history of multi-agent tool calls in the format
\texttt{Agent:Tool$\to$Result}, followed by a query
\texttt{?Agent:Tool=}.  The model must retrieve the correct result
from the context.  We evaluate two difficulty tiers:
\textbf{EASY}~(10--30 calls, $T_{\max}{=}400$) and
\textbf{HARD}~(60--120 calls, $T_{\max}{=}1{,}500$).
Training runs for 2{,}500 steps with evaluation every 250 steps.

\paragraph{Experiment~2: Resource Economy.}
Multi-agent economy sequences where agents mine, trade, and build
with gold and wood.  The query requires predicting an agent's current
resource count.
\textbf{EASY}: 5--15 steps, 3~agents, no noise gap ($T_{\max}{=}400$).
\textbf{HARD}: 40--100 steps, 6~agents, noise gap of 200--400 tokens
($T_{\max}{=}3{,}000$).
Training runs for 2{,}500 steps with evaluation every 250 steps.

\paragraph{Experiment~3: System Prompt Amnesia.}
Two sub-tasks testing retrieval of system-level variables across
a long intervening gap:
\begin{itemize}
  \item \textbf{Sub-task A (CoT Retrieval)}: $n_{\text{vars}}{=}4$
        system variables followed by a noisy chain-of-thought scratchpad.
        EASY: gap 10--30 tokens ($T_{\max}{=}400$);
        HARD: gap 150--300 tokens ($T_{\max}{=}3{,}000$).
  \item \textbf{Sub-task B (Specific Recall)}: $n_{\text{vars}}{=}6$
        system variables followed by confusing distractors.
        EASY: gap 10--30 tokens ($T_{\max}{=}500$);
        HARD: gap 80--200 tokens ($T_{\max}{=}2{,}500$).
\end{itemize}
Training runs for 6{,}000 steps with evaluation every 250 steps.
All models in this experiment use $L{=}4$ layers and $H{=}4$ heads.

\subsection{Recurrence Implementation}

The selective state-space recurrence
$\mathbf{H}_t = \alpha_t \mathbf{H}_{t-1} + \beta_t \mathbf{U}_{t-1} + \gamma_t \mathbf{U}_t$ is computed via a chunked SSD algorithm (following Mamba-2) with chunk size~$C{=}64$.
The intra-chunk computation uses lower-triangular decay matrices via segment-sum, while boundary states are propagated across chunks via a short sequential scan over $T/C$ steps.
All scan-critical arithmetic (log-space cumulative sums, decay matrices, boundary propagation) is performed in FP32 to ensure numerical stability; only the final output is cast back to the training precision.

\begin{figure*}[!htbp]
  \centering
  \includegraphics[width=\linewidth]{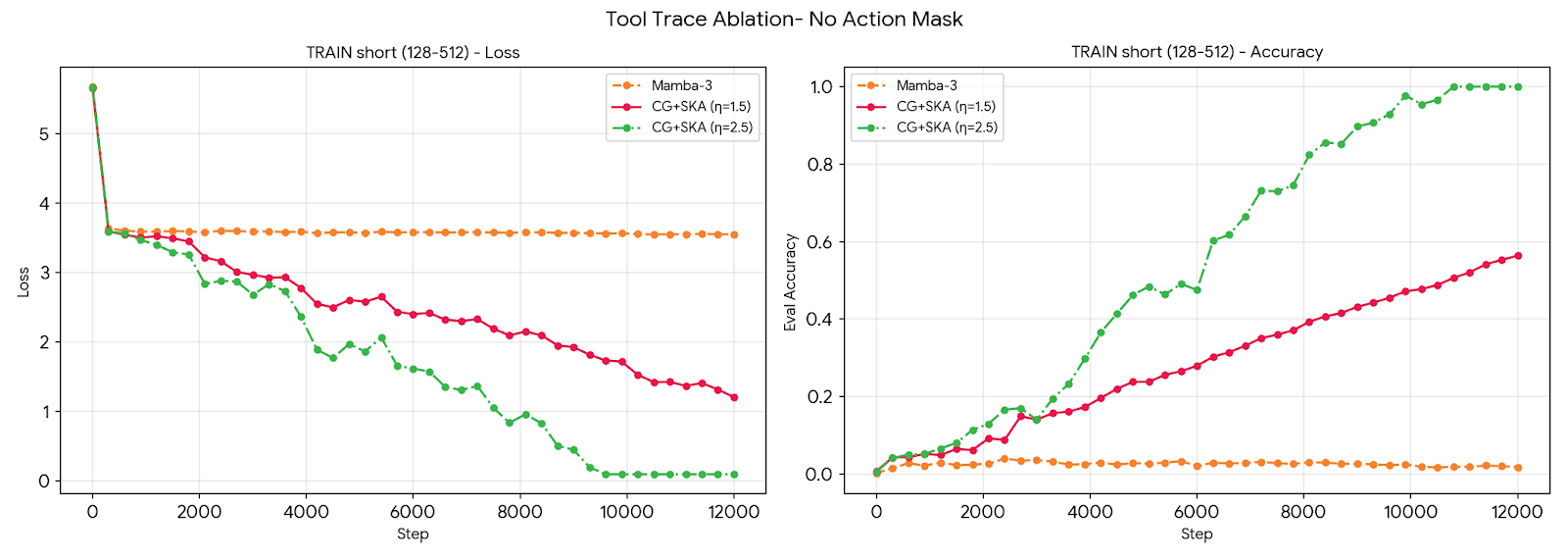}
  \caption{Training with action mask set to 1 for all inputs on the Tool Trace Delayed Commit tasks (short). We find that CG+SKA works fine without the action mask. The x-axis is for training steps and y-axis for log loss (left) and accuracy (right).}
  \label{fig:tool-no-action-mask}
\end{figure*}
\begin{figure*}[ht!]
  \centering
  \includegraphics[width=\linewidth]{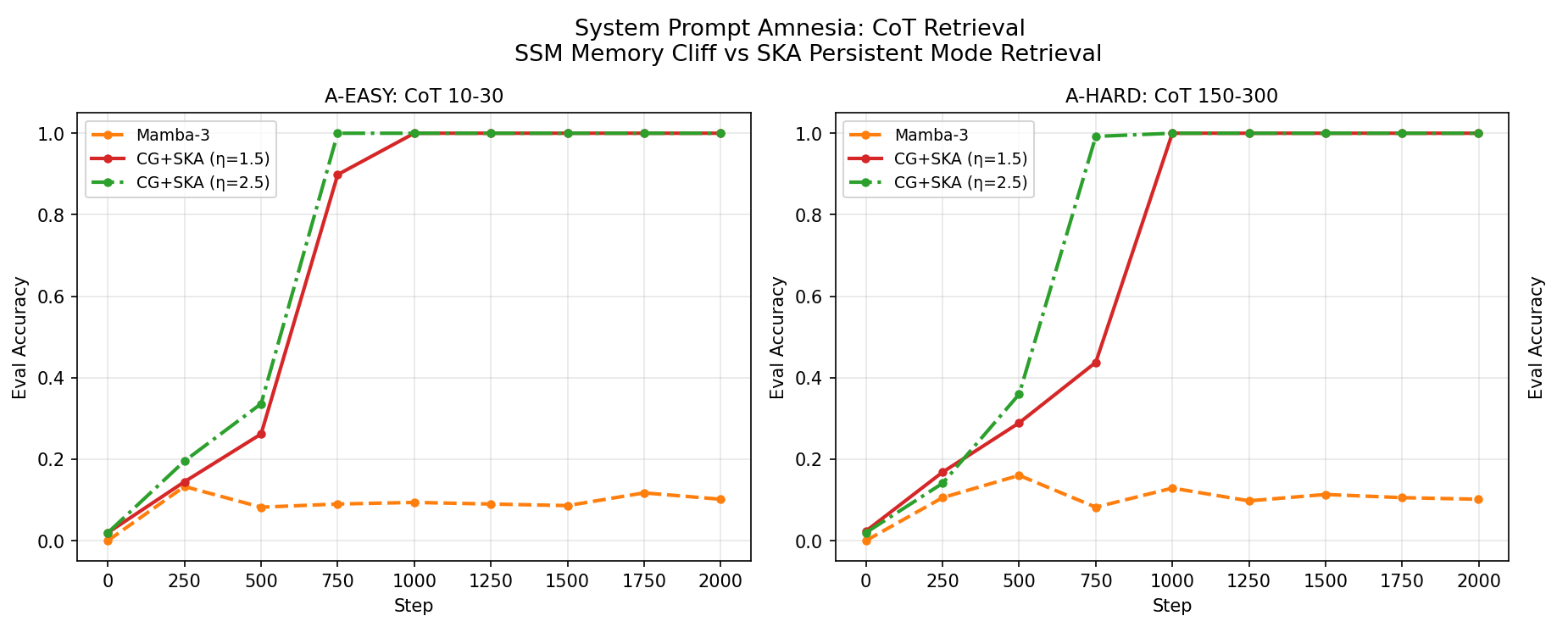}
  \caption{A Training run for 400k parameter Mamba-3 variants over the easy and hard System Prompt CoT Retrieval benchmark. The x-axis  represents training iterations and y-axis represents training iterations and y-curve evaluation accuracy. Mamba-3 is unable to learn on this task.}
  \label{fig:cot-experiment}
\end{figure*}
\section{Additional Results}
\subsection{Ablation: Removing the Action Mask}
Section~\ref{app:pipeline} introduced the action mask as a form of structured inductive bias that restricts SKA's covariance accumulation to fact-bearing tokens, preventing dilution from uninformative distractor spans.
A natural question is whether this mask is necessary for SKA to function, or whether the spectral filtering mechanism is robust enough to isolate retrievable content even when distractor tokens contribute to the covariance statistics.
Figure~\ref{fig:tool-no-action-mask} answers this by training CG+SKA on the short Tool Trace tier with the action mask set to $m_t=1$ for all input tokens, so that distractor tokens and fact-bearing tokens contribute equally to G, M, and C.
Both the log loss and accuracy curves closely track those of the action-masked configuration, and the model converges to comparable accuracy.

\subsection{Additional plots for System Prompt CoT Retrieval task}
\label{sec:cot-appendix}
Figure~\ref{fig:cot-experiment} shows training curves for all three model variants on the Easy and Hard tiers of the CoT Retrieval sub-task. On the Easy tier (left panel), both CG+SKA configurations rise sharply from near-chance accuracy to 100\% within the first 500 steps and remain saturated for the remainder of training. Mamba-3, despite training for the full 2,000 steps, fluctuates between 10--20\% accuracy with no discernible upward trend, confirming that the recurrent state cannot learn to preserve system-prompt bindings even across short scratchpad gaps.

The Hard tier (right panel) presents a more demanding setting, with 150--300 CoT steps interposed between the system prompt and the query. The CG+SKA variants exhibit a brief delay relative to the Easy tier---reaching saturation at roughly 750--1,000 steps rather than 500---but converge to the same perfect accuracy. This delay is expected: longer sequences produce larger covariance matrices with more distractor contribution, requiring slightly more training for the Koopman operator's spectral structure to stabilize. Mamba-3 again shows no learning signal, remaining at or below 16\% throughout. The flat Mamba-3 curves on both tiers, contrasted with the rapid SKA convergence, provide further evidence that the retrieval gain is architectural rather than a consequence of additional training signal or capacity.

\section{SKA, Ridge Retrieval, and the Comparison to Attention}
\label{app:attention_regression_clean}

\subsection{Setup}

We work with a single retrieval head. Let
\[
k_t \in \mathbb{R}^r,
\qquad
v_t \in \mathbb{R}^P,
\qquad
q \in \mathbb{R}^r
\]
denote the normalized key, value, and query features. Stack the keys and values into
\[
K =
\begin{bmatrix}
k_1^\top \\
\vdots \\
k_T^\top
\end{bmatrix}
\in \mathbb{R}^{T \times r},
\qquad
V =
\begin{bmatrix}
v_1^\top \\
\vdots \\
v_T^\top
\end{bmatrix}
\in \mathbb{R}^{T \times P}.
\]
Define the Gram matrix, ridge-regularized Gram matrix, value-key covariance, and lag-one covariance by
\begin{align}
G &\coloneqq K^\top K = \sum_{t=1}^T k_t k_t^\top \in \mathbb{R}^{r \times r}, \\
\tilde G &\coloneqq G + \varepsilon I_r, \qquad \varepsilon > 0, \\
C_v &\coloneqq V^\top K = \sum_{t=1}^T v_t k_t^\top \in \mathbb{R}^{P \times r}, \\
M &\coloneqq \sum_{t=1}^{T-1} k_{t+1} k_t^\top \in \mathbb{R}^{r \times r}.
\end{align}
Let
\[
\tilde G = L L^\top
\]
be the Cholesky factorization of \(\tilde G\), where \(L\) is lower triangular with positive diagonal.

We define three retrieval maps on the same feature space:
\begin{align}
\hat y_{\mathrm{lin}}(q) &\coloneqq C_v q, \\
\hat y_{\mathrm{ridge}}(q) &\coloneqq C_v \tilde G^{-1} q, \\
\hat y_{\mathrm{SKA}}(q) &\coloneqq \eta\, C_v \tilde G^{-1} \Phi_K q,
\end{align}
where \(\eta > 0\) is a learned scalar and
\begin{equation}
\Phi_K \coloneqq L \hat A_w^K L^{-1}.
\end{equation}
Here
\begin{equation}
A_w \coloneqq L^{-1} M L^{-\top}
\end{equation}
is the whitened lag-one operator, and \(\hat A_w\) denotes the operator actually used by the retrieval head after spectral normalization and optional variance restoration.

Throughout this section, vector norms are Euclidean norms and matrix norms are operator norms unless a subscript \(F\) is written.

\subsection{The Ridge Retrieval Problem}

The natural supervised retrieval objective is
\begin{equation}
\mathcal L(B)
\coloneqq
\sum_{t=1}^T \|v_t - B k_t\|_2^2 + \varepsilon \|B\|_F^2,
\qquad
B \in \mathbb{R}^{P \times r}.
\label{eq:clean_retrieval_objective}
\end{equation}

\begin{lemma}[Positive definiteness of the regularized Gram matrix]
\label{lem:clean_gram_pd}
For every \(\varepsilon > 0\), the matrix \(\tilde G = G + \varepsilon I_r\) is symmetric positive definite. In particular,
\[
\lambda_{\min}(\tilde G) \ge \varepsilon,
\]
so \(\tilde G\) is invertible and its Cholesky factorization exists.
\end{lemma}

\begin{proof}
The matrix \(G\) is symmetric because
\[
G^\top = (K^\top K)^\top = K^\top K = G.
\]
For any \(u \in \mathbb{R}^r\),
\[
u^\top G u
=
u^\top \left(\sum_{t=1}^T k_t k_t^\top\right) u
=
\sum_{t=1}^T u^\top k_t k_t^\top u
=
\sum_{t=1}^T (k_t^\top u)^2
\ge 0.
\]
Therefore \(G\) is positive semidefinite. Hence, for any nonzero \(u \in \mathbb{R}^r\),
\[
u^\top \tilde G u
=
u^\top G u + \varepsilon u^\top I_r u
=
u^\top G u + \varepsilon \|u\|_2^2
\ge
\varepsilon \|u\|_2^2
>
0.
\]
So \(\tilde G\) is positive definite. The lower bound on \(\lambda_{\min}(\tilde G)\) follows immediately.
\end{proof}

\begin{proposition}[Closed-form ridge minimizer]
\label{prop:clean_ridge_minimizer}
The unique minimizer of \(\mathcal L(B)\) is
\[
B^* = C_v \tilde G^{-1}.
\]
\end{proposition}

\begin{proof}
We expand the objective step by step. For each \(t\),
\[
\|v_t - B k_t\|_2^2
=
(v_t - B k_t)^\top (v_t - B k_t).
\]
Expanding the product gives
\[
(v_t - B k_t)^\top (v_t - B k_t)
=
v_t^\top v_t - 2 v_t^\top B k_t + k_t^\top B^\top B k_t.
\]
Summing over \(t\),
\begin{align}
\sum_{t=1}^T \|v_t - B k_t\|_2^2
&=
\sum_{t=1}^T v_t^\top v_t
-
2 \sum_{t=1}^T v_t^\top B k_t
+
\sum_{t=1}^T k_t^\top B^\top B k_t.
\end{align}
We rewrite each term in trace form.

For the first term,
\[
\sum_{t=1}^T v_t^\top v_t = \operatorname{tr}(V^\top V).
\]

For the second term,
\begin{align}
\sum_{t=1}^T v_t^\top B k_t
&=
\sum_{t=1}^T \operatorname{tr}(v_t^\top B k_t) \\
&=
\sum_{t=1}^T \operatorname{tr}(B k_t v_t^\top) \\
&=
\operatorname{tr}\left(B \sum_{t=1}^T k_t v_t^\top\right).
\end{align}
Since
\[
\sum_{t=1}^T k_t v_t^\top = (C_v)^\top,
\]
this becomes
\[
\sum_{t=1}^T v_t^\top B k_t = \operatorname{tr}(B C_v^\top).
\]

For the third term,
\begin{align}
\sum_{t=1}^T k_t^\top B^\top B k_t
&=
\sum_{t=1}^T \operatorname{tr}(k_t^\top B^\top B k_t) \\
&=
\sum_{t=1}^T \operatorname{tr}(B k_t k_t^\top B^\top) \\
&=
\operatorname{tr}\left(B \left(\sum_{t=1}^T k_t k_t^\top\right) B^\top\right) \\
&=
\operatorname{tr}(B G B^\top).
\end{align}

The regularization term is
\[
\varepsilon \|B\|_F^2 = \varepsilon \operatorname{tr}(B B^\top).
\]
Therefore
\begin{align}
\mathcal L(B)
&=
\operatorname{tr}(V^\top V)
-
2 \operatorname{tr}(B C_v^\top)
+
\operatorname{tr}(B G B^\top)
+
\varepsilon \operatorname{tr}(B B^\top) \\
&=
\operatorname{tr}(V^\top V)
-
2 \operatorname{tr}(B C_v^\top)
+
\operatorname{tr}\!\bigl(B (G + \varepsilon I_r) B^\top\bigr) \\
&=
\operatorname{tr}(V^\top V)
-
2 \operatorname{tr}(B C_v^\top)
+
\operatorname{tr}(B \tilde G B^\top).
\end{align}

Now differentiate with respect to \(B\). Since \(\tilde G\) is symmetric,
\[
\frac{\partial}{\partial B} \operatorname{tr}(B \tilde G B^\top)
=
2 B \tilde G.
\]
Also,
\[
\frac{\partial}{\partial B} \operatorname{tr}(B C_v^\top)
=
C_v.
\]
Hence
\[
\nabla_B \mathcal L(B)
=
-2 C_v + 2 B \tilde G.
\]
Setting the gradient to zero gives
\[
-2 C_v + 2 B \tilde G = 0.
\]
Dividing by \(2\),
\[
B \tilde G = C_v.
\]
Right multiplying by \(\tilde G^{-1}\) yields
\[
B = C_v \tilde G^{-1}.
\]
This critical point is unique because \(\tilde G \succ 0\), so the objective is strictly convex in \(B\).
\end{proof}

\begin{proposition}[Exact excess-risk identity]
\label{prop:clean_excess_risk}
For every \(B \in \mathbb{R}^{P \times r}\),
\[
\mathcal L(B) - \mathcal L(B^*)
=
\|(B - B^*) \tilde G^{1/2}\|_F^2.
\]
In particular, \(B^*\) is the unique minimizer of \(\mathcal L\).
\end{proposition}

\begin{proof}
From the previous proof,
\[
\mathcal L(B)
=
\operatorname{tr}(V^\top V)
-
2 \operatorname{tr}(B C_v^\top)
+
\operatorname{tr}(B \tilde G B^\top).
\]
Because \(B^* \tilde G = C_v\), we have
\[
C_v^\top = \tilde G B^{*\top}.
\]
Substituting this into the middle term,
\[
\operatorname{tr}(B C_v^\top)
=
\operatorname{tr}(B \tilde G B^{*\top}).
\]
Therefore
\[
\mathcal L(B)
=
\operatorname{tr}(V^\top V)
-
2 \operatorname{tr}(B \tilde G B^{*\top})
+
\operatorname{tr}(B \tilde G B^\top).
\]
Likewise,
\[
\mathcal L(B^*)
=
\operatorname{tr}(V^\top V)
-
2 \operatorname{tr}(B^* \tilde G B^{*\top})
+
\operatorname{tr}(B^* \tilde G B^{*\top}),
\]
so
\[
\mathcal L(B^*)
=
\operatorname{tr}(V^\top V)
-
\operatorname{tr}(B^* \tilde G B^{*\top}).
\]
Subtracting,
\begin{align}
\mathcal L(B) - \mathcal L(B^*)
&=
\operatorname{tr}(B \tilde G B^\top)
-
2 \operatorname{tr}(B \tilde G B^{*\top})
+
\operatorname{tr}(B^* \tilde G B^{*\top}).
\end{align}
Now expand
\[
(B - B^*) \tilde G (B - B^*)^\top
=
B \tilde G B^\top
-
B \tilde G B^{*\top}
-
B^* \tilde G B^\top
+
B^* \tilde G B^{*\top}.
\]
Taking traces and using \(\operatorname{tr}(X) = \operatorname{tr}(X^\top)\),
\[
\operatorname{tr}(B^* \tilde G B^\top)
=
\operatorname{tr}\!\bigl((B^* \tilde G B^\top)^\top\bigr)
=
\operatorname{tr}(B \tilde G B^{*\top}),
\]
since \(\tilde G\) is symmetric. Thus
\[
\operatorname{tr}\!\bigl((B - B^*) \tilde G (B - B^*)^\top\bigr)
=
\operatorname{tr}(B \tilde G B^\top)
-
2 \operatorname{tr}(B \tilde G B^{*\top})
+
\operatorname{tr}(B^* \tilde G B^{*\top}).
\]
Comparing with the expression above,
\[
\mathcal L(B) - \mathcal L(B^*)
=
\operatorname{tr}\!\bigl((B - B^*) \tilde G (B - B^*)^\top\bigr).
\]
Since \(\tilde G \succ 0\), it has a unique positive semidefinite square root \(\tilde G^{1/2}\), and
\[
\operatorname{tr}\!\bigl((B - B^*) \tilde G (B - B^*)^\top\bigr)
=
\|(B - B^*) \tilde G^{1/2}\|_F^2.
\]
This proves the identity.
\end{proof}

\subsection{Whitening}

\begin{proposition}[Whitened and unwhitened lag-one operators are similar]
\label{prop:clean_similarity}
Define
\[
A_\varepsilon \coloneqq M \tilde G^{-1},
\qquad
A_w \coloneqq L^{-1} M L^{-\top}.
\]
Then
\[
A_w = L^{-1} A_\varepsilon L.
\]
Consequently \(A_w\) and \(A_\varepsilon\) have the same eigenvalues.
\end{proposition}

\begin{proof}
Since \(\tilde G = L L^\top\),
\[
\tilde G^{-1} = (L L^\top)^{-1} = L^{-\top} L^{-1}.
\]
Therefore
\[
A_\varepsilon
=
M \tilde G^{-1}
=
M L^{-\top} L^{-1}.
\]
Left multiplying by \(L^{-1}\) and right multiplying by \(L\) gives
\begin{align}
L^{-1} A_\varepsilon L
&=
L^{-1} (M L^{-\top} L^{-1}) L \\
&=
L^{-1} M L^{-\top} \\
&=
A_w.
\end{align}
Hence \(A_w\) is similar to \(A_\varepsilon\), so they have the same eigenvalues.
\end{proof}

\subsection{Linear Attention as an Incomplete Solve}

\begin{proposition}[Linear attention is one unpreconditioned gradient step]
\label{prop:clean_lin_one_step}
The linear attention map
\[
\hat y_{\mathrm{lin}}(q) = C_v q
\]
is exactly the prediction produced by one gradient step on \(\mathcal L(B)\) from the initialization \(B_0 = 0\) with step size \(\alpha = \tfrac{1}{2}\).
\end{proposition}

\begin{proof}
From Proposition~\ref{prop:clean_ridge_minimizer},
\[
\nabla_B \mathcal L(B) = -2 C_v + 2 B \tilde G.
\]
At \(B_0 = 0\),
\[
\nabla_B \mathcal L(B_0) = -2 C_v.
\]
A gradient step of size \(\alpha = \tfrac{1}{2}\) gives
\begin{align}
B_1
&=
B_0 - \alpha \nabla_B \mathcal L(B_0) \\
&=
0 - \frac{1}{2}(-2 C_v) \\
&=
C_v.
\end{align}
Applying this linear map to the query,
\[
B_1 q = C_v q = \hat y_{\mathrm{lin}}(q).
\]
\end{proof}

\begin{theorem}[Linear attention is close to ridge exactly when the Gram is close to isotropic]
\label{thm:clean_lin_vs_ridge}
Let
\[
\lambda_{\min} \coloneqq \lambda_{\min}(\tilde G),
\qquad
\lambda_{\max} \coloneqq \lambda_{\max}(\tilde G),
\qquad
\kappa(\tilde G) \coloneqq \frac{\lambda_{\max}}{\lambda_{\min}}.
\]
Then
\[
\min_{a > 0}
\|\hat y_{\mathrm{lin}}(q) - a \hat y_{\mathrm{ridge}}(q)\|_2
\le
\|C_v\|\,
\frac{\kappa(\tilde G)-1}{\kappa(\tilde G)+1}\,
\|q\|_2.
\]
The minimizing scalar is
\[
a^*
=
\frac{2}{\lambda_{\max}^{-1} + \lambda_{\min}^{-1}}
=
\frac{2 \lambda_{\min} \lambda_{\max}}{\lambda_{\min} + \lambda_{\max}}.
\]
In particular, if \(\tilde G = c I_r\), then choosing \(a = c\) yields exact equality
\[
\hat y_{\mathrm{lin}}(q) = c \hat y_{\mathrm{ridge}}(q).
\]
\end{theorem}

\begin{proof}
We begin from the definitions
\[
\hat y_{\mathrm{lin}}(q) = C_v q,
\qquad
\hat y_{\mathrm{ridge}}(q) = C_v \tilde G^{-1} q.
\]
For any scalar \(a > 0\),
\begin{align}
\hat y_{\mathrm{lin}}(q) - a \hat y_{\mathrm{ridge}}(q)
&=
C_v q - a C_v \tilde G^{-1} q \\
&=
C_v (I_r - a \tilde G^{-1}) q.
\end{align}
Taking norms,
\[
\|\hat y_{\mathrm{lin}}(q) - a \hat y_{\mathrm{ridge}}(q)\|_2
\le
\|C_v\| \, \|I_r - a \tilde G^{-1}\| \, \|q\|_2.
\]
So it remains to minimize \(\|I_r - a \tilde G^{-1}\|\) over \(a > 0\).

Since \(\tilde G\) is symmetric positive definite, it is orthogonally diagonalizable:
\[
\tilde G = U \Lambda U^\top,
\qquad
\Lambda = \operatorname{diag}(\lambda_1,\dots,\lambda_r),
\qquad
\lambda_i > 0.
\]
Hence
\[
\tilde G^{-1} = U \Lambda^{-1} U^\top,
\]
and therefore
\[
I_r - a \tilde G^{-1}
=
U (I_r - a \Lambda^{-1}) U^\top.
\]
Because orthogonal similarity preserves operator norm,
\[
\|I_r - a \tilde G^{-1}\|
=
\|I_r - a \Lambda^{-1}\|
=
\max_{1 \le i \le r} \left|1 - \frac{a}{\lambda_i}\right|.
\]
The function \(\lambda \mapsto 1 - a/\lambda\) is monotone increasing in \(\lambda > 0\). Hence the maximum absolute value over the interval \([\lambda_{\min}, \lambda_{\max}]\) is attained at one of the endpoints:
\[
\|I_r - a \tilde G^{-1}\|
=
\max\left\{
\left|1 - \frac{a}{\lambda_{\min}}\right|,
\left|1 - \frac{a}{\lambda_{\max}}\right|
\right\}.
\]
The minimizing value \(a^*\) is obtained by equalizing the endpoint errors with opposite signs:
\[
1 - \frac{a}{\lambda_{\max}}
=
-\left(1 - \frac{a}{\lambda_{\min}}\right).
\]
Solving this equation,
\begin{align}
1 - \frac{a}{\lambda_{\max}}
&=
-1 + \frac{a}{\lambda_{\min}}, \\
2
&=
a\left(\frac{1}{\lambda_{\min}} + \frac{1}{\lambda_{\max}}\right), \\
a^*
&=
\frac{2}{\lambda_{\min}^{-1} + \lambda_{\max}^{-1}}
=
\frac{2 \lambda_{\min} \lambda_{\max}}{\lambda_{\min} + \lambda_{\max}}.
\end{align}
Substituting \(a^*\),
\begin{align}
\left|1 - \frac{a^*}{\lambda_{\max}}\right|
&=
1 - \frac{2 \lambda_{\min}}{\lambda_{\min} + \lambda_{\max}}
=
\frac{\lambda_{\max} - \lambda_{\min}}{\lambda_{\max} + \lambda_{\min}}, \\
\left|1 - \frac{a^*}{\lambda_{\min}}\right|
&=
\frac{2 \lambda_{\max}}{\lambda_{\min} + \lambda_{\max}} - 1
=
\frac{\lambda_{\max} - \lambda_{\min}}{\lambda_{\max} + \lambda_{\min}}.
\end{align}
So
\[
\min_{a > 0} \|I_r - a \tilde G^{-1}\|
=
\frac{\lambda_{\max} - \lambda_{\min}}{\lambda_{\max} + \lambda_{\min}}
=
\frac{\kappa(\tilde G)-1}{\kappa(\tilde G)+1}.
\]
Combining with the earlier norm bound proves the theorem.

If \(\tilde G = c I_r\), then \(\tilde G^{-1} = c^{-1} I_r\), so
\[
\hat y_{\mathrm{ridge}}(q) = C_v c^{-1} q.
\]
Thus
\[
c \hat y_{\mathrm{ridge}}(q) = C_v q = \hat y_{\mathrm{lin}}(q).
\]
\end{proof}

\subsection{SKA as an Exact Ridge Solve at \(K=0\)}

\begin{proposition}[Two triangular solves produce the ridge value map]
\label{prop:clean_two_solves}
Let \(\tilde G = L L^\top\). Define \(Y\) and \(Z\) by the two triangular systems
\begin{align}
L Y &= C_v^\top, \\
L^\top Z &= Y.
\end{align}
Then
\[
B_v \coloneqq Z^\top = C_v \tilde G^{-1} = B^*.
\]
\end{proposition}

\begin{proof}
From the first triangular system,
\[
L Y = C_v^\top.
\]
Since \(L\) is invertible,
\[
Y = L^{-1} C_v^\top.
\]
Now substitute this into the second system:
\[
L^\top Z = Y = L^{-1} C_v^\top.
\]
Again \(L^\top\) is invertible, so
\[
Z = L^{-\top} Y = L^{-\top} L^{-1} C_v^\top.
\]
Because
\[
\tilde G^{-1} = (L L^\top)^{-1} = L^{-\top} L^{-1},
\]
we obtain
\[
Z = \tilde G^{-1} C_v^\top.
\]
Transpose both sides:
\[
Z^\top = C_v \tilde G^{-1}.
\]
Therefore
\[
B_v = Z^\top = C_v \tilde G^{-1} = B^*.
\]
\end{proof}

\begin{proposition}[Full SKA retrieval formula]
\label{prop:clean_full_ska}
Let
\[
q_w \coloneqq L^{-1} q,
\qquad
w_f \coloneqq \hat A_w^K q_w,
\qquad
z_f \coloneqq L w_f.
\]
Then
\[
\hat y_{\mathrm{SKA}}(q)
=
\eta\, C_v \tilde G^{-1} \Phi_K q,
\qquad
\Phi_K = L \hat A_w^K L^{-1}.
\]
\end{proposition}

\begin{proof}
We proceed step by step.

First, whiten the query:
\[
q_w = L^{-1} q.
\]

Second, apply the power filter:
\[
w_f = \hat A_w^K q_w.
\]

Third, unwhiten:
\begin{align}
z_f
&=
L w_f \\
&=
L \hat A_w^K q_w \\
&=
L \hat A_w^K L^{-1} q \\
&=
\Phi_K q.
\end{align}

Fourth, apply the value map from Proposition~\ref{prop:clean_two_solves}:
\begin{align}
\hat y_{\mathrm{SKA}}(q)
&=
\eta\, B_v z_f \\
&=
\eta\, (C_v \tilde G^{-1}) z_f \\
&=
\eta\, C_v \tilde G^{-1} \Phi_K q.
\end{align}
This is the claimed formula.
\end{proof}

\begin{corollary}[SKA with \(K=0\) equals ridge retrieval exactly]
\label{cor:clean_k0_equals_ridge}
If \(K=0\), then \(\Phi_0 = I_r\), and therefore
\[
\hat y_{\mathrm{SKA}}(q)
=
\eta\, \hat y_{\mathrm{ridge}}(q).
\]
\end{corollary}

\begin{proof}
If \(K=0\), then \(\hat A_w^0 = I_r\). Hence
\[
\Phi_0 = L I_r L^{-1} = I_r.
\]
Substituting into Proposition~\ref{prop:clean_full_ska},
\[
\hat y_{\mathrm{SKA}}(q)
=
\eta\, C_v \tilde G^{-1} I_r q
=
\eta\, C_v \tilde G^{-1} q
=
\eta\, \hat y_{\mathrm{ridge}}(q).
\]
\end{proof}

\subsection{The Exact Bias Introduced by the Power Filter}

\begin{proposition}[Exact query-side bias of the power filter]
\label{prop:clean_power_bias}
Define
\[
B_K \coloneqq B^* \Phi_K = C_v \tilde G^{-1} \Phi_K.
\]
Then for every query \(q\),
\[
\frac{1}{\eta}\hat y_{\mathrm{SKA}}(q) - \hat y_{\mathrm{ridge}}(q)
=
B^*(\Phi_K - I_r) q.
\]
Moreover,
\[
\mathcal L(B_K) - \mathcal L(B^*)
=
\|B^*(\Phi_K - I_r)\tilde G^{1/2}\|_F^2.
\]
\end{proposition}

\begin{proof}
From the definitions,
\[
\frac{1}{\eta}\hat y_{\mathrm{SKA}}(q)
=
B_K q
=
B^* \Phi_K q.
\]
Subtract \(\hat y_{\mathrm{ridge}}(q) = B^* q\):
\begin{align}
\frac{1}{\eta}\hat y_{\mathrm{SKA}}(q) - \hat y_{\mathrm{ridge}}(q)
&=
B^* \Phi_K q - B^* q \\
&=
B^*(\Phi_K - I_r) q.
\end{align}
This proves the first identity.

For the second identity, apply Proposition~\ref{prop:clean_excess_risk} with \(B = B_K\):
\[
\mathcal L(B_K) - \mathcal L(B^*)
=
\|(B_K - B^*) \tilde G^{1/2}\|_F^2.
\]
Since
\[
B_K - B^*
=
B^* \Phi_K - B^*
=
B^*(\Phi_K - I_r),
\]
we obtain
\[
\mathcal L(B_K) - \mathcal L(B^*)
=
\|B^*(\Phi_K - I_r)\tilde G^{1/2}\|_F^2.
\]
\end{proof}

\subsection{When the Power Filter Helps}

\begin{theorem}[Persistent and transient regime bound]
\label{thm:clean_persistent_transient}
Assume that \(\hat A_w\) is normal, so that
\[
\hat A_w = U \Lambda U^*,
\qquad
\Lambda = \operatorname{diag}(\lambda_1,\dots,\lambda_r),
\]
with \(U^* U = I_r\). Let
\[
q_w \coloneqq L^{-1} q.
\]
Decompose \(q_w\) orthogonally as
\[
q_w = q_{\mathrm{pers}} + q_{\mathrm{trans}},
\]
where \(q_{\mathrm{pers}}\) lies in the span of eigenvectors whose eigenvalues satisfy
\[
|1 - \lambda_i| \le \delta,
\]
and \(q_{\mathrm{trans}}\) lies in the span of eigenvectors whose eigenvalues satisfy
\[
|\lambda_i| \le \rho < 1.
\]
Define the persistent target
\[
y_{\mathrm{pers}} \coloneqq \eta\, C_v L^{-\top} q_{\mathrm{pers}}.
\]
Then
\[
\left\|
\frac{1}{\eta}\hat y_{\mathrm{SKA}}(q)
-
C_v L^{-\top} q_{\mathrm{pers}}
\right\|_2
\le
\|C_v L^{-\top}\|
\left(
\bigl((1+\delta)^K - 1\bigr)\|q_{\mathrm{pers}}\|_2
+
\rho^K \|q_{\mathrm{trans}}\|_2
\right).
\]
\end{theorem}

\begin{proof}
By Proposition~\ref{prop:clean_full_ska},
\[
\frac{1}{\eta}\hat y_{\mathrm{SKA}}(q)
=
C_v \tilde G^{-1} \Phi_K q.
\]
Because \(\tilde G^{-1} = L^{-\top} L^{-1}\) and \(\Phi_K = L \hat A_w^K L^{-1}\),
\begin{align}
C_v \tilde G^{-1} \Phi_K q
&=
C_v L^{-\top} L^{-1} L \hat A_w^K L^{-1} q \\
&=
C_v L^{-\top} \hat A_w^K q_w.
\end{align}
Hence
\[
\frac{1}{\eta}\hat y_{\mathrm{SKA}}(q)
=
C_v L^{-\top} \hat A_w^K q_w.
\]
Substitute the decomposition \(q_w = q_{\mathrm{pers}} + q_{\mathrm{trans}}\):
\[
\frac{1}{\eta}\hat y_{\mathrm{SKA}}(q)
=
C_v L^{-\top} \hat A_w^K q_{\mathrm{pers}}
+
C_v L^{-\top} \hat A_w^K q_{\mathrm{trans}}.
\]
Subtract \(C_v L^{-\top} q_{\mathrm{pers}}\):
\begin{align}
\frac{1}{\eta}\hat y_{\mathrm{SKA}}(q)
-
C_v L^{-\top} q_{\mathrm{pers}}
&=
C_v L^{-\top} (\hat A_w^K q_{\mathrm{pers}} - q_{\mathrm{pers}})
+
C_v L^{-\top} \hat A_w^K q_{\mathrm{trans}} \\
&=
C_v L^{-\top} (\hat A_w^K - I_r) q_{\mathrm{pers}}
+
C_v L^{-\top} \hat A_w^K q_{\mathrm{trans}}.
\end{align}
Taking norms and using the triangle inequality,
\begin{align}
\left\|
\frac{1}{\eta}\hat y_{\mathrm{SKA}}(q)
-
C_v L^{-\top} q_{\mathrm{pers}}
\right\|_2
&\le
\|C_v L^{-\top}\|
\left(
\|(\hat A_w^K - I_r) q_{\mathrm{pers}}\|_2
+
\|\hat A_w^K q_{\mathrm{trans}}\|_2
\right).
\end{align}
So it remains to bound the two terms inside the parentheses.

First consider the transient part. Since \(q_{\mathrm{trans}}\) lies in the invariant subspace where \(|\lambda_i| \le \rho\), and since \(\hat A_w\) is normal, the operator norm of \(\hat A_w^K\) restricted to this subspace is the maximum modulus of the corresponding eigenvalues raised to the \(K\)-th power. Therefore
\[
\|\hat A_w^K q_{\mathrm{trans}}\|_2
\le
\rho^K \|q_{\mathrm{trans}}\|_2.
\]

Now consider the persistent part. Write
\[
q_{\mathrm{pers}} = \sum_{i \in \mathcal P} \alpha_i u_i,
\]
where \(u_i\) are orthonormal eigenvectors of \(\hat A_w\) and \(\mathcal P\) indexes the persistent subspace. Then
\[
(\hat A_w^K - I_r) q_{\mathrm{pers}}
=
\sum_{i \in \mathcal P} \alpha_i (\lambda_i^K - 1) u_i.
\]
Hence
\[
\|(\hat A_w^K - I_r) q_{\mathrm{pers}}\|_2
\le
\max_{i \in \mathcal P} |\lambda_i^K - 1| \, \|q_{\mathrm{pers}}\|_2.
\]
We now bound \(|\lambda_i^K - 1|\). Factor:
\[
\lambda_i^K - 1
=
(\lambda_i - 1)\sum_{j=0}^{K-1} \lambda_i^j.
\]
Taking absolute values,
\[
|\lambda_i^K - 1|
\le
|\lambda_i - 1|
\sum_{j=0}^{K-1} |\lambda_i|^j.
\]
Because \(|1 - \lambda_i| \le \delta\), we have \(|\lambda_i| \le 1 + \delta\), so
\[
|\lambda_i^K - 1|
\le
\delta \sum_{j=0}^{K-1} (1+\delta)^j.
\]
The geometric sum evaluates to
\[
\sum_{j=0}^{K-1} (1+\delta)^j
=
\frac{(1+\delta)^K - 1}{(1+\delta)-1}
=
\frac{(1+\delta)^K - 1}{\delta}.
\]
Therefore
\[
|\lambda_i^K - 1|
\le
(1+\delta)^K - 1.
\]
Hence
\[
\|(\hat A_w^K - I_r) q_{\mathrm{pers}}\|_2
\le
\bigl((1+\delta)^K - 1\bigr)\|q_{\mathrm{pers}}\|_2.
\]

Substituting the persistent and transient bounds into the earlier inequality yields
\[
\left\|
\frac{1}{\eta}\hat y_{\mathrm{SKA}}(q)
-
C_v L^{-\top} q_{\mathrm{pers}}
\right\|_2
\le
\|C_v L^{-\top}\|
\left(
\bigl((1+\delta)^K - 1\bigr)\|q_{\mathrm{pers}}\|_2
+
\rho^K \|q_{\mathrm{trans}}\|_2
\right).
\]
This proves the theorem.
\end{proof}

\begin{corollary}[A sufficient condition under which \(K>0\) improves on \(K=0\) for a persistent target]
\label{cor:clean_filter_helps}
Under the assumptions of Theorem~\ref{thm:clean_persistent_transient}, the unfiltered rule \(K=0\) satisfies
\[
\left\|
\frac{1}{\eta}\hat y_{\mathrm{SKA},\,K=0}(q)
-
C_v L^{-\top} q_{\mathrm{pers}}
\right\|_2
\le
\|C_v L^{-\top}\| \, \|q_{\mathrm{trans}}\|_2.
\]
Therefore the \(K>0\) power filter has a strictly smaller upper bound whenever
\[
\bigl((1+\delta)^K - 1\bigr)\|q_{\mathrm{pers}}\|_2
<
(1-\rho^K)\|q_{\mathrm{trans}}\|_2.
\]
\end{corollary}

\begin{proof}
When \(K=0\), we have \(\hat A_w^0 = I_r\). Therefore
\[
\frac{1}{\eta}\hat y_{\mathrm{SKA},\,K=0}(q)
=
C_v L^{-\top} q_w
=
C_v L^{-\top}(q_{\mathrm{pers}} + q_{\mathrm{trans}}).
\]
Subtract the persistent target:
\[
\frac{1}{\eta}\hat y_{\mathrm{SKA},\,K=0}(q)
-
C_v L^{-\top} q_{\mathrm{pers}}
=
C_v L^{-\top} q_{\mathrm{trans}}.
\]
Taking norms gives
\[
\left\|
\frac{1}{\eta}\hat y_{\mathrm{SKA},\,K=0}(q)
-
C_v L^{-\top} q_{\mathrm{pers}}
\right\|_2
\le
\|C_v L^{-\top}\| \, \|q_{\mathrm{trans}}\|_2.
\]
The \(K>0\) upper bound is given by Theorem~\ref{thm:clean_persistent_transient}. Comparing the two bounds, the \(K>0\) upper bound is strictly smaller whenever
\[
\bigl((1+\delta)^K - 1\bigr)\|q_{\mathrm{pers}}\|_2
+
\rho^K \|q_{\mathrm{trans}}\|_2
<
\|q_{\mathrm{trans}}\|_2.
\]
Rearranging,
\[
\bigl((1+\delta)^K - 1\bigr)\|q_{\mathrm{pers}}\|_2
<
(1-\rho^K)\|q_{\mathrm{trans}}\|_2.
\]
\end{proof}

\subsection{A Strict Expressivity Difference}

\begin{proposition}[Fixed-statistic SKA is linear in the query, softmax attention is not]
\label{prop:clean_expressivity}
Conditioned on fixed prefix statistics, the map
\[
q \mapsto \hat y_{\mathrm{SKA}}(q)
\]
is linear in \(q\), whereas a softmax retrieval map of the form
\[
q \mapsto V^\top \operatorname{softmax}(K \Theta q)
\]
is generally nonlinear in \(q\). Consequently, there exist single-head softmax retrieval rules that cannot be represented exactly by any fixed-statistic SKA head.
\end{proposition}

\begin{proof}
For fixed \(\eta\), \(C_v\), \(\tilde G\), and \(\Phi_K\),
\[
\hat y_{\mathrm{SKA}}(q)
=
\eta\, C_v \tilde G^{-1} \Phi_K q.
\]
This has the form \(A q\) for the fixed matrix
\[
A \coloneqq \eta\, C_v \tilde G^{-1} \Phi_K,
\]
so it is linear in \(q\).

Now consider a two-token, one-dimensional example with
\[
K =
\begin{bmatrix}
1 \\
-1
\end{bmatrix},
\qquad
V =
\begin{bmatrix}
1 \\
0
\end{bmatrix},
\qquad
\Theta = [\beta],
\qquad
\beta > 0.
\]
Then for a scalar query \(q \in \mathbb{R}\),
\[
K \Theta q
=
\begin{bmatrix}
\beta q \\
-\beta q
\end{bmatrix}.
\]
Applying softmax,
\[
\operatorname{softmax}(K \Theta q)
=
\begin{bmatrix}
\frac{e^{\beta q}}{e^{\beta q} + e^{-\beta q}} \\
\frac{e^{-\beta q}}{e^{\beta q} + e^{-\beta q}}
\end{bmatrix}.
\]
Multiplying by \(V^\top = [1 \; 0]\) yields the scalar output
\[
f(q)
=
V^\top \operatorname{softmax}(K \Theta q)
=
\frac{e^{\beta q}}{e^{\beta q} + e^{-\beta q}}
=
\frac{1}{1 + e^{-2\beta q}}.
\]
Differentiate:
\[
f'(q)
=
\frac{2\beta e^{-2\beta q}}{(1 + e^{-2\beta q})^2}.
\]
This derivative depends on \(q\), so \(f\) is not linear on any open interval. Therefore no linear map \(A q\) can agree with this softmax rule on any open interval. Since fixed-statistic SKA is always linear in \(q\), this softmax rule cannot be represented exactly by any fixed-statistic SKA head.
\end{proof}

\subsection{Exact Memory Crossover}

\begin{proposition}[When SKA uses less memory than explicit key-value caching]
\label{prop:clean_memory_crossover}
Suppose an explicit-cache attention mechanism stores all \(T\) keys and values for one head, using
\[
T(r + P)
\]
floating-point numbers. The SKA head stores
\[
2r^2 + Pr + r + 1
\]
floating-point numbers, namely \(G\), \(M\), \(C_v\), the previous key, and one scalar. Therefore SKA uses less memory exactly when
\[
T > \frac{2r^2 + Pr + r + 1}{r + P}.
\]
\end{proposition}

\begin{proof}
The explicit-cache attention mechanism stores \(T\) key vectors in \(\mathbb{R}^r\) and \(T\) value vectors in \(\mathbb{R}^P\). Hence the total number of stored scalars is
\[
Tr + TP = T(r+P).
\]
The SKA head stores
\[
G \in \mathbb{R}^{r \times r},
\qquad
M \in \mathbb{R}^{r \times r},
\qquad
C_v \in \mathbb{R}^{P \times r},
\qquad
k_{t-1} \in \mathbb{R}^r,
\]
and one scalar. Therefore the total number of stored scalars is
\[
r^2 + r^2 + Pr + r + 1 = 2r^2 + Pr + r + 1.
\]
SKA uses less memory precisely when
\[
2r^2 + Pr + r + 1 < T(r+P).
\]
Dividing both sides by \(r+P > 0\) gives
\[
T > \frac{2r^2 + Pr + r + 1}{r + P}.
\]
\end{proof}

\subsection{Regime Summary}

The results above imply the following comparison.

First, linear attention is competitive when \(\tilde G\) is close to a scalar multiple of the identity. In that regime, Theorem~\ref{thm:clean_lin_vs_ridge} shows that the missing inverse is almost just a scalar calibration.

Second, ridge retrieval and SKA with \(K=0\) are preferable when the key Gram matrix is anisotropic. In that regime, linear attention has an unavoidable error proportional to
\[
\frac{\kappa(\tilde G)-1}{\kappa(\tilde G)+1},
\]
while SKA computes the ridge inverse explicitly.

Third, SKA with \(K>0\) is preferable when answer-bearing signal lies in persistent modes and nuisance mass lies in transient modes. Theorem~\ref{thm:clean_persistent_transient} and Corollary~\ref{cor:clean_filter_helps} quantify exactly when the spectral filter improves the persistent target approximation.

Fourth, softmax attention can outperform fixed-statistic SKA when the desired retrieval rule is genuinely nonlinear in the query. Proposition~\ref{prop:clean_expressivity} makes this precise.

Finally, for sufficiently long contexts, SKA has a strict memory advantage over explicit key-value caching by Proposition~\ref{prop:clean_memory_crossover}.
\section{Gradient Flow Analysis}
\label{sec:gradient-bounds}

We analyze gradient propagation through the nonlinear branches of the
sequence and MLP sublayers. Every model considered here uses residual
connections, so the full block Jacobian has the form
\[
J_{\mathrm{block}} = I + J_{\mathrm{branch}}.
\]
The identity term is common to both standard baselines and Koopman-LM.
To isolate the effect of the nonlinear mechanism itself, we first analyze
the branch Jacobians. The residual connection can then be added back at
the end.

Throughout this section we use column-vector convention. For a scalar loss
\(\mathcal L\), gradients are column vectors. For a differentiable map
\(y = f(x)\), the Jacobian is
\[
J_f(x) = \frac{\partial y}{\partial x},
\]
and backpropagation uses
\[
\nabla_x \mathcal L = J_f(x)^\top \nabla_y \mathcal L.
\]

\subsection{The LM head bottleneck}

Let \(H \in \mathbb{R}^{C \times D}\) be the matrix of hidden states at
\(C\) token positions, let \(W \in \mathbb{R}^{V \times D}\) be the LM head,
and let the logits be
\[
L = H W^\top \in \mathbb{R}^{C \times V}.
\]
Let
\[
G_L \coloneqq \nabla_L \mathcal L \in \mathbb{R}^{C \times V}
\]
denote the logit gradient.

\begin{proposition}[LM head rank bottleneck]
\label{prop:grad-head}
The hidden-state gradient after the LM head is
\[
G_H \coloneqq \nabla_H \mathcal L = G_L W \in \mathbb{R}^{C \times D},
\]
and satisfies
\[
\operatorname{rank}(G_H) \le \operatorname{rank}(W) \le D.
\]
\end{proposition}

\begin{proof}
Since
\[
L = H W^\top,
\]
the differential is
\[
dL = dH \, W^\top.
\]
By standard matrix backpropagation,
\[
\nabla_H \mathcal L = \nabla_L \mathcal L \, W = G_L W.
\]
Now use the rank inequality
\[
\operatorname{rank}(AB) \le \min\{\operatorname{rank}(A), \operatorname{rank}(B)\}.
\]
Applying this with \(A = G_L\) and \(B = W\) gives
\[
\operatorname{rank}(G_H)
=
\operatorname{rank}(G_L W)
\le
\min\{\operatorname{rank}(G_L), \operatorname{rank}(W)\}
\le
\operatorname{rank}(W)
\le
D.
\]
\end{proof}

Proposition~\ref{prop:grad-head} is architecture-independent. Every model
must begin backpropagation from a gradient that lies in a subspace of
dimension at most \(D\).

\subsection{Softmax attention}

We now analyze a single attention head at one query position, with keys
and values treated as fixed during the local Jacobian computation.

Let
\[
K =
\begin{bmatrix}
k_1^\top \\
\vdots \\
k_T^\top
\end{bmatrix}
\in \mathbb{R}^{T \times r},
\qquad
V =
\begin{bmatrix}
v_1^\top \\
\vdots \\
v_T^\top
\end{bmatrix}
\in \mathbb{R}^{T \times P},
\qquad
q \in \mathbb{R}^{r}.
\]
Define the logits, attention weights, and output by
\begin{align}
s(q) &\coloneqq \frac{K q}{\sqrt{d_k}} \in \mathbb{R}^{T}, \\
\alpha(q) &\coloneqq \operatorname{softmax}(s(q)) \in \mathbb{R}^{T}, \\
y_{\mathrm{attn}}(q) &\coloneqq V^\top \alpha(q) \in \mathbb{R}^{P}.
\end{align}

\begin{proposition}[Exact softmax-attention Jacobian]
\label{prop:grad-softmax-jac}
The Jacobian of a single attention head with respect to the query is
\[
J_{\mathrm{attn}}(q)
=
\frac{\partial y_{\mathrm{attn}}}{\partial q}
=
\frac{1}{\sqrt{d_k}} \, V^\top S(q) K,
\]
where
\[
S(q) \coloneqq \operatorname{Diag}(\alpha(q)) - \alpha(q)\alpha(q)^\top
\in \mathbb{R}^{T \times T}.
\]
Moreover:
\begin{align}
\operatorname{rank}(J_{\mathrm{attn}}(q))
&\le
\min\{P, r, T-1\}, \label{eq:grad-softmax-rank} \\
\|J_{\mathrm{attn}}(q)\|_2
&\le
\frac{\|V\|_2 \, \|K\|_2}{2\sqrt{d_k}}.
\label{eq:grad-softmax-op}
\end{align}
\end{proposition}

\begin{proof}
We proceed step by step.

First, the logits are
\[
s(q) = \frac{Kq}{\sqrt{d_k}}.
\]
Therefore their differential is
\[
ds = \frac{K \, dq}{\sqrt{d_k}}.
\]

Second, we differentiate the softmax map. For each \(i,j\),
\[
\alpha_i = \frac{e^{s_i}}{\sum_{\ell=1}^{T} e^{s_\ell}},
\]
so
\[
\frac{\partial \alpha_i}{\partial s_j}
=
\alpha_i (\delta_{ij} - \alpha_j).
\]
Hence the Jacobian of softmax is
\[
\frac{\partial \alpha}{\partial s}
=
\operatorname{Diag}(\alpha) - \alpha \alpha^\top
=
S(q).
\]
Thus
\[
d\alpha = S(q)\, ds.
\]

Third, the attention output is
\[
y_{\mathrm{attn}}(q) = V^\top \alpha(q),
\]
so
\[
dy_{\mathrm{attn}} = V^\top \, d\alpha = V^\top S(q)\, ds.
\]
Substituting the expression for \(ds\),
\[
dy_{\mathrm{attn}}
=
V^\top S(q)\frac{K\,dq}{\sqrt{d_k}}
=
\left(\frac{1}{\sqrt{d_k}}V^\top S(q)K\right) dq.
\]
Therefore
\[
J_{\mathrm{attn}}(q)
=
\frac{1}{\sqrt{d_k}}V^\top S(q)K.
\]

We now prove the rank bound. Since
\[
J_{\mathrm{attn}}(q)
=
\frac{1}{\sqrt{d_k}}V^\top S(q)K,
\]
the rank inequality for products gives
\[
\operatorname{rank}(J_{\mathrm{attn}}(q))
\le
\min\{\operatorname{rank}(V^\top), \operatorname{rank}(S(q)), \operatorname{rank}(K)\}.
\]
Now
\[
\operatorname{rank}(V^\top) \le P,
\qquad
\operatorname{rank}(K) \le r.
\]
Also
\[
S(q)\mathbf 1
=
\operatorname{Diag}(\alpha)\mathbf 1 - \alpha \alpha^\top \mathbf 1
=
\alpha - \alpha(\mathbf 1^\top \alpha)
=
\alpha - \alpha
=
0,
\]
so \(S(q)\) has a nontrivial null vector \(\mathbf 1\), hence
\[
\operatorname{rank}(S(q)) \le T-1.
\]
Combining these inequalities gives
\[
\operatorname{rank}(J_{\mathrm{attn}}(q))
\le
\min\{P, r, T-1\}.
\]

We now prove the operator-norm bound. By submultiplicativity,
\[
\|J_{\mathrm{attn}}(q)\|_2
\le
\frac{1}{\sqrt{d_k}}
\|V^\top\|_2 \, \|S(q)\|_2 \, \|K\|_2
=
\frac{1}{\sqrt{d_k}}
\|V\|_2 \, \|S(q)\|_2 \, \|K\|_2.
\]
So it remains to bound \(\|S(q)\|_2\).

The matrix \(S(q)\) is symmetric. For row \(i\), the diagonal entry is
\[
S_{ii} = \alpha_i(1-\alpha_i),
\]
and for \(j \ne i\),
\[
S_{ij} = -\alpha_i \alpha_j.
\]
Hence the absolute row sum of row \(i\) is
\begin{align}
\sum_{j=1}^{T} |S_{ij}|
&=
\alpha_i(1-\alpha_i) + \sum_{j \ne i} \alpha_i \alpha_j \\
&=
\alpha_i(1-\alpha_i) + \alpha_i \sum_{j \ne i} \alpha_j \\
&=
\alpha_i(1-\alpha_i) + \alpha_i(1-\alpha_i) \\
&=
2\alpha_i(1-\alpha_i).
\end{align}
Since \(0 \le \alpha_i \le 1\), we have
\[
2\alpha_i(1-\alpha_i) \le \frac{1}{2},
\]
because the quadratic \(x(1-x)\) attains its maximum \(1/4\) at \(x=1/2\).
Therefore
\[
\|S(q)\|_\infty \le \frac{1}{2}.
\]
Since \(S(q)\) is symmetric,
\[
\|S(q)\|_2 \le \|S(q)\|_\infty \le \frac{1}{2}.
\]
Substituting into the previous bound yields
\[
\|J_{\mathrm{attn}}(q)\|_2
\le
\frac{\|V\|_2 \, \|K\|_2}{2\sqrt{d_k}}.
\]
This proves the proposition.
\end{proof}

\begin{corollary}[Softmax saturation suppresses the query gradient]
\label{cor:grad-softmax-saturation}
Let
\[
m(q) \coloneqq \max_{1 \le i \le T} \alpha_i(q).
\]
Then
\[
\|S(q)\|_2 \le 2(1-m(q)),
\]
and therefore
\[
\|J_{\mathrm{attn}}(q)\|_2
\le
\frac{2(1-m(q))}{\sqrt{d_k}}
\|V\|_2 \, \|K\|_2.
\]
In particular, if \(m(q) \to 1\), then \(\|J_{\mathrm{attn}}(q)\|_2 \to 0\).
If \(\alpha(q)\) is exactly one-hot, then \(J_{\mathrm{attn}}(q)=0\).
\end{corollary}

\begin{proof}
For row \(i\), the absolute row sum of \(S(q)\) is
\[
2\alpha_i(1-\alpha_i).
\]
Let \(m = m(q)\). If \(\alpha_i = m\), then
\[
2\alpha_i(1-\alpha_i) = 2m(1-m) \le 2(1-m).
\]
If \(\alpha_i \ne m\), then \(\alpha_i \le 1-m\), so
\[
2\alpha_i(1-\alpha_i) \le 2\alpha_i \le 2(1-m).
\]
Thus every row sum is at most \(2(1-m)\), which implies
\[
\|S(q)\|_\infty \le 2(1-m).
\]
Since \(S(q)\) is symmetric,
\[
\|S(q)\|_2 \le \|S(q)\|_\infty \le 2(1-m).
\]
Substituting this into the bound from
Proposition~\ref{prop:grad-softmax-jac} gives
\[
\|J_{\mathrm{attn}}(q)\|_2
\le
\frac{2(1-m(q))}{\sqrt{d_k}}
\|V\|_2 \, \|K\|_2.
\]
If \(m(q)\to 1\), the right-hand side tends to \(0\). If \(\alpha(q)\) is
exactly one-hot, then \(\operatorname{Diag}(\alpha)=\alpha\alpha^\top\), so
\(S(q)=0\), hence \(J_{\mathrm{attn}}(q)=0\).
\end{proof}

Corollary~\ref{cor:grad-softmax-saturation} is the precise version of the
informal statement that highly concentrated attention suppresses query-side
gradient flow. What vanishes is not a heuristic “effective rank” in general,
but the actual softmax covariance factor \(S(q)\). 

\subsection{SKA retrieval}

We now analyze the query-side Jacobian of a single SKA head, treating the
prefix statistics as fixed.

Let
\[
\tilde G = G + \varepsilon I_r = L L^\top,
\qquad
B_v = C_v \tilde G^{-1},
\qquad
\Phi_K = L \hat A_w^K L^{-1},
\]
where \(L\) is the Cholesky factor of \(\tilde G\),
\(\hat A_w\) is the spectrally normalized whitened Koopman operator, and
\(K \in \mathbb{N}\) is the power-filter order. Let the head output be
\[
y_{\mathrm{SKA}}(q) = \eta \, B_v \Phi_K q
=
\eta \, B_v L \hat A_w^K L^{-1} q.
\]

\begin{proposition}[Exact SKA Jacobian and norm bound]
\label{prop:grad-ska-jac}
The Jacobian of one SKA head with respect to the query is
\[
J_{\mathrm{SKA}}(q)
=
\frac{\partial y_{\mathrm{SKA}}}{\partial q}
=
\eta \, B_v L \hat A_w^K L^{-1}.
\]
Moreover, if
\[
\|\hat A_w\|_2 \le \gamma,
\]
then
\[
\|J_{\mathrm{SKA}}(q)\|_2
\le
\eta \, \|B_v\|_2 \, \gamma^K
\sqrt{\frac{\|G\|_2 + \varepsilon}{\varepsilon}}.
\]
\end{proposition}

\begin{proof}
The map
\[
q \mapsto y_{\mathrm{SKA}}(q) = \eta \, B_v L \hat A_w^K L^{-1} q
\]
is linear in \(q\), so its Jacobian is the matrix multiplying \(q\):
\[
J_{\mathrm{SKA}}(q)
=
\eta \, B_v L \hat A_w^K L^{-1}.
\]

We now bound its operator norm. By submultiplicativity,
\[
\|J_{\mathrm{SKA}}(q)\|_2
\le
\eta \, \|B_v\|_2 \, \|L\|_2 \, \|\hat A_w^K\|_2 \, \|L^{-1}\|_2.
\]
Since \(\|\hat A_w\|_2 \le \gamma\),
\[
\|\hat A_w^K\|_2 \le \|\hat A_w\|_2^K \le \gamma^K.
\]

Next, because \(\tilde G = L L^\top\) is symmetric positive definite, the
singular values of \(L\) are the square roots of the eigenvalues of
\(\tilde G\). Therefore
\[
\|L\|_2 = \sqrt{\lambda_{\max}(\tilde G)},
\qquad
\|L^{-1}\|_2 = \frac{1}{\sqrt{\lambda_{\min}(\tilde G)}}.
\]
Also,
\[
\lambda_{\max}(\tilde G)
=
\lambda_{\max}(G + \varepsilon I_r)
\le
\|G\|_2 + \varepsilon,
\]
and
\[
\lambda_{\min}(\tilde G) \ge \varepsilon.
\]
Hence
\[
\|L\|_2 \le \sqrt{\|G\|_2 + \varepsilon},
\qquad
\|L^{-1}\|_2 \le \frac{1}{\sqrt{\varepsilon}}.
\]
Substituting these bounds gives
\[
\|J_{\mathrm{SKA}}(q)\|_2
\le
\eta \, \|B_v\|_2 \, \gamma^K
\sqrt{\frac{\|G\|_2 + \varepsilon}{\varepsilon}}.
\]
\end{proof}

\begin{proposition}[Full-rank regime for SKA]
\label{prop:grad-ska-full-rank}
Assume all of the following:
\begin{enumerate}
    \item \(P \ge r\),
    \item \(\operatorname{rank}(C_v)=r\),
    \item \(\hat A_w\) is invertible.
\end{enumerate}
Then
\[
\operatorname{rank}(J_{\mathrm{SKA}}(q)) = r.
\]
In particular, as a map from query directions to head outputs, the SKA
branch has no query-side null direction.
\end{proposition}

\begin{proof}
Since \(\tilde G\) is positive definite, \(\tilde G^{-1}\) is invertible.
Therefore
\[
B_v = C_v \tilde G^{-1}
\]
has the same rank as \(C_v\), namely
\[
\operatorname{rank}(B_v)=\operatorname{rank}(C_v)=r.
\]
Since \(P \ge r\), this means \(B_v\) has full column rank.

Next, \(L\) and \(L^{-1}\) are invertible. By assumption, \(\hat A_w\) is
invertible, so \(\hat A_w^K\) is invertible. Hence
\[
\Phi_K = L \hat A_w^K L^{-1}
\]
is invertible.

Now
\[
J_{\mathrm{SKA}}(q) = \eta \, B_v \Phi_K.
\]
Multiplication by the nonzero scalar \(\eta\) does not change rank, and
right multiplication by an invertible matrix does not change rank. Hence
\[
\operatorname{rank}(J_{\mathrm{SKA}}(q))
=
\operatorname{rank}(B_v)
=
r.
\]
\end{proof}

\begin{proposition}[Lower singular-value bound on the active query subspace]
\label{prop:grad-ska-smin}
Assume the hypotheses of Proposition~\ref{prop:grad-ska-full-rank}, and set
\[
\nu \coloneqq \sigma_{\min}(\hat A_w) > 0.
\]
Then
\[
\sigma_{\min}(J_{\mathrm{SKA}}(q))
\ge
\eta \, \sigma_{\min}(B_v) \, \nu^K
\sqrt{\frac{\varepsilon}{\|G\|_2 + \varepsilon}}.
\]
\end{proposition}

\begin{proof}
We start from
\[
J_{\mathrm{SKA}}(q)
=
\eta \, B_v L \hat A_w^K L^{-1}.
\]
For matrices of compatible dimensions,
\[
\sigma_{\min}(AB) \ge \sigma_{\min}(A)\sigma_{\min}(B).
\]
Applying this repeatedly,
\[
\sigma_{\min}(J_{\mathrm{SKA}}(q))
\ge
\eta \, \sigma_{\min}(B_v)\sigma_{\min}(L)\sigma_{\min}(\hat A_w^K)\sigma_{\min}(L^{-1}).
\]

Since \(\hat A_w\) is invertible,
\[
\sigma_{\min}(\hat A_w^K) \ge \sigma_{\min}(\hat A_w)^K = \nu^K.
\]

Next,
\[
\sigma_{\min}(L) = \sqrt{\lambda_{\min}(\tilde G)} \ge \sqrt{\varepsilon}.
\]
Also,
\[
\sigma_{\min}(L^{-1})
=
\frac{1}{\sigma_{\max}(L)}
=
\frac{1}{\sqrt{\lambda_{\max}(\tilde G)}}
\ge
\frac{1}{\sqrt{\|G\|_2 + \varepsilon}}.
\]
Substituting these bounds gives
\[
\sigma_{\min}(J_{\mathrm{SKA}}(q))
\ge
\eta \, \sigma_{\min}(B_v) \, \nu^K
\sqrt{\frac{\varepsilon}{\|G\|_2 + \varepsilon}}.
\]
\end{proof}

Propositions~\ref{prop:grad-ska-full-rank} and
\ref{prop:grad-ska-smin} are the strongest correct forms of the
“no extra projection” intuition. The unconditional statement is false in
general, because \(B_v\) can be rank-deficient and \(\hat A_w\) can be
singular. The correct statement is that, in the full-rank regime, the SKA
query map remains full rank and its nonzero singular values are controlled
by conditioning parameters, not by a softmax simplex factor. :contentReference[oaicite:2]{index=2}

\subsection{Koopman MLP}

We now analyze the spectral core of the Koopman MLP. For the nonlinear branch,
write
\begin{align}
a &= W_{\mathrm{lift}} h \in \mathbb{R}^{d_k}, \\
u &= \operatorname{SiLU}(a) \in \mathbb{R}^{d_k}, \\
z &= R u \in \mathbb{R}^{d_k}, \\
y_{\mathrm{KMLP}}(h) &= W_{\mathrm{read}} z \in \mathbb{R}^{d},
\end{align}
where
\[
R = \operatorname{blockdiag}(R_1,\dots,R_{d_k/2}),
\qquad
R_i =
\begin{bmatrix}
\gamma_i & \omega_i \\
-\omega_i & \gamma_i
\end{bmatrix}.
\]
Define
\[
\lambda_i = \gamma_i + i\omega_i,
\qquad
\rho_i = |\lambda_i| = \sqrt{\gamma_i^2 + \omega_i^2}.
\]
By construction, \(\rho_i \le 1\).

\begin{proposition}[Each spectral block is isotropic scaling times rotation]
\label{prop:grad-kmlp-block}
For each \(i\), if \(\rho_i > 0\), then
\[
R_i
=
\rho_i
\begin{bmatrix}
\cos \theta_i & \sin \theta_i \\
-\sin \theta_i & \cos \theta_i
\end{bmatrix}
\]
for some angle \(\theta_i\). Consequently,
\[
R_i^\top R_i = \rho_i^2 I_2,
\]
so both singular values of \(R_i\) are equal to \(\rho_i\).
Therefore
\[
\|R\|_2 = \max_i \rho_i \le 1.
\]
\end{proposition}

\begin{proof}
If \(\rho_i > 0\), define
\[
\cos \theta_i = \frac{\gamma_i}{\rho_i},
\qquad
\sin \theta_i = \frac{\omega_i}{\rho_i}.
\]
Then
\[
R_i
=
\begin{bmatrix}
\gamma_i & \omega_i \\
-\omega_i & \gamma_i
\end{bmatrix}
=
\rho_i
\begin{bmatrix}
\cos \theta_i & \sin \theta_i \\
-\sin \theta_i & \cos \theta_i
\end{bmatrix}.
\]
The second factor is an orthogonal rotation matrix. Hence
\[
R_i^\top R_i
=
\rho_i^2 I_2.
\]
The singular values of \(R_i\) are the square roots of the eigenvalues of
\(R_i^\top R_i\), so both singular values equal \(\rho_i\).

Since \(R\) is block diagonal, its singular values are the union of the
singular values of the blocks. Therefore
\[
\|R\|_2 = \max_i \rho_i \le 1.
\]
\end{proof}

\begin{proposition}[Koopman MLP Jacobian]
\label{prop:grad-kmlp-jac}
Let
\[
D_{\mathrm{SiLU}}(a) \coloneqq \operatorname{Diag}(\operatorname{SiLU}'(a)).
\]
Then the Jacobian of the Koopman MLP nonlinear branch is
\[
J_{\mathrm{KMLP}}(h)
=
W_{\mathrm{read}} R D_{\mathrm{SiLU}}(a) W_{\mathrm{lift}}.
\]
Moreover, if
\[
c_{\mathrm{SiLU}} \coloneqq \sup_{x \in \mathbb{R}} |\operatorname{SiLU}'(x)| < \infty,
\]
then
\[
\|J_{\mathrm{KMLP}}(h)\|_2
\le
c_{\mathrm{SiLU}}
\|W_{\mathrm{read}}\|_2
\|W_{\mathrm{lift}}\|_2.
\]
\end{proposition}

\begin{proof}
We differentiate the branch step by step.

First,
\[
a = W_{\mathrm{lift}} h,
\]
so
\[
da = W_{\mathrm{lift}} \, dh.
\]

Second,
\[
u = \operatorname{SiLU}(a),
\]
so
\[
du = D_{\mathrm{SiLU}}(a) \, da.
\]

Third,
\[
z = R u,
\]
so
\[
dz = R \, du.
\]

Fourth,
\[
y_{\mathrm{KMLP}} = W_{\mathrm{read}} z,
\]
so
\[
dy_{\mathrm{KMLP}} = W_{\mathrm{read}} \, dz.
\]

Combining these relations,
\begin{align}
dy_{\mathrm{KMLP}}
&=
W_{\mathrm{read}} R D_{\mathrm{SiLU}}(a) W_{\mathrm{lift}} \, dh.
\end{align}
Therefore
\[
J_{\mathrm{KMLP}}(h)
=
W_{\mathrm{read}} R D_{\mathrm{SiLU}}(a) W_{\mathrm{lift}}.
\]

For the operator-norm bound, apply submultiplicativity:
\[
\|J_{\mathrm{KMLP}}(h)\|_2
\le
\|W_{\mathrm{read}}\|_2
\|R\|_2
\|D_{\mathrm{SiLU}}(a)\|_2
\|W_{\mathrm{lift}}\|_2.
\]
By Proposition~\ref{prop:grad-kmlp-block},
\[
\|R\|_2 \le 1.
\]
Also, since \(D_{\mathrm{SiLU}}(a)\) is diagonal,
\[
\|D_{\mathrm{SiLU}}(a)\|_2
=
\max_j |\operatorname{SiLU}'(a_j)|
\le
c_{\mathrm{SiLU}}.
\]
Hence
\[
\|J_{\mathrm{KMLP}}(h)\|_2
\le
c_{\mathrm{SiLU}}
\|W_{\mathrm{read}}\|_2
\|W_{\mathrm{lift}}\|_2.
\]
\end{proof}

The important correction is that the spectral core is not orthogonal unless
\(\rho_i=1\) for every block. The correct statement is that each block is an
isotropic scaling by \(\rho_i\) followed by a rotation, with \(\rho_i \le 1\).
So the Koopman block does not introduce anisotropy within a \(2 \times 2\)
mode pair, but it can contract that pair uniformly. :contentReference[oaicite:3]{index=3}

\subsection{Layerwise comparison}

We now collect the previous bounds into one depth-wise statement.

Let the gradient immediately after the LM head be
\[
g_{\mathrm{head}}.
\]
Suppose the gradient from layer \(m\) down to layer \(\ell\) passes through
branch Jacobians
\[
J_m, J_{m-1}, \dots, J_{\ell+1}.
\]
Then
\[
g_\ell = J_{\ell+1}^\top J_{\ell+2}^\top \cdots J_m^\top g_{\mathrm{head}}.
\]

\begin{theorem}[Depth-wise gradient upper bounds]
\label{thm:grad-depth}
For any such sequence of Jacobians,
\[
\|g_\ell\|_2
\le
\left(\prod_{j=\ell+1}^{m} \|J_j\|_2\right)\|g_{\mathrm{head}}\|_2.
\]
Consequently:

\begin{enumerate}
    \item For a standard architecture, if \(\mathcal A_{\ell}\) is the set of
    attention branches above layer \(\ell\), and \(\mathcal R_{\ell}\) is the
    set of all other branch Jacobians above layer \(\ell\), then
    \[
    \|g_\ell\|_2
    \le
    \left(
    \prod_{j \in \mathcal A_{\ell}}
    \frac{\|V_j\|_2 \|K_j\|_2}{2\sqrt{d_{k,j}}}
    \right)
    \left(
    \prod_{j \in \mathcal R_{\ell}}
    \|J_j\|_2
    \right)
    \|g_{\mathrm{head}}\|_2.
    \]

    \item More sharply, if \(m_j = \max_i \alpha_{j,i}\) for the \(j\)-th
    attention branch, then
    \[
    \|g_\ell\|_2
    \le
    \left(
    \prod_{j \in \mathcal A_{\ell}}
    \frac{2(1-m_j)\|V_j\|_2 \|K_j\|_2}{\sqrt{d_{k,j}}}
    \right)
    \left(
    \prod_{j \in \mathcal R_{\ell}}
    \|J_j\|_2
    \right)
    \|g_{\mathrm{head}}\|_2.
    \]

    \item For Koopman-LM, if \(\mathcal S_{\ell}\) is the set of SKA branches
    above layer \(\ell\), \(\mathcal M_{\ell}\) is the set of Koopman MLP
    branches above layer \(\ell\), and \(\mathcal R_{\ell}\) is the set of all
    remaining branch Jacobians above layer \(\ell\), then
    \[
    \|g_\ell\|_2
    \le
    \left(
    \prod_{j \in \mathcal S_{\ell}}
    \eta_j \|B_{v,j}\|_2 \gamma_j^{K_j}
    \sqrt{\frac{\|G_j\|_2+\varepsilon_j}{\varepsilon_j}}
    \right)
    \left(
    \prod_{j \in \mathcal M_{\ell}}
    c_{\mathrm{SiLU}}
    \|W_{\mathrm{read},j}\|_2
    \|W_{\mathrm{lift},j}\|_2
    \right)
    \left(
    \prod_{j \in \mathcal R_{\ell}}
    \|J_j\|_2
    \right)
    \|g_{\mathrm{head}}\|_2.
    \]
\end{enumerate}
\end{theorem}

\begin{proof}
The general inequality follows from repeated use of
\[
\|A^\top x\|_2 \le \|A^\top\|_2 \|x\|_2 = \|A\|_2 \|x\|_2.
\]
Indeed,
\begin{align}
\|g_\ell\|_2
&=
\|J_{\ell+1}^\top J_{\ell+2}^\top \cdots J_m^\top g_{\mathrm{head}}\|_2 \\
&\le
\|J_{\ell+1}\|_2
\|J_{\ell+2}\|_2
\cdots
\|J_m\|_2
\|g_{\mathrm{head}}\|_2 \\
&=
\left(\prod_{j=\ell+1}^{m} \|J_j\|_2\right)\|g_{\mathrm{head}}\|_2.
\end{align}

The standard-architecture bounds follow by substituting the softmax bound
from Proposition~\ref{prop:grad-softmax-jac} and the sharper saturation
bound from Corollary~\ref{cor:grad-softmax-saturation} for every attention
branch.

The Koopman-LM bound follows by substituting the SKA bound from
Proposition~\ref{prop:grad-ska-jac} for every SKA branch and the Koopman MLP
bound from Proposition~\ref{prop:grad-kmlp-jac} for every Koopman MLP branch.
\end{proof}

\subsection{Interpretation}

The conclusion is the following.

After the LM head, every architecture already starts from a gradient of
rank at most \(D\), by Proposition~\ref{prop:grad-head}. Beyond that shared
bottleneck, standard attention introduces an additional softmax covariance
factor
\[
S(q) = \operatorname{Diag}(\alpha(q)) - \alpha(q)\alpha(q)^\top
\]
at each attention branch. This factor always has rank at most \(T-1\), and
its operator norm goes to zero as attention saturates by
Corollary~\ref{cor:grad-softmax-saturation}.

By contrast, an SKA branch contributes a fixed linear operator
\[
J_{\mathrm{SKA}}(q) = \eta B_v L \hat A_w^K L^{-1},
\]
whose conditioning is controlled by the ridge parameter \(\varepsilon\), the
Gram spectrum, the value map \(B_v\), and the spectral radius bound on
\(\hat A_w\). Under the full-rank conditions in
Propositions~\ref{prop:grad-ska-full-rank} and \ref{prop:grad-ska-smin},
this operator remains full rank as a query-to-output map.

The Koopman MLP spectral core contributes a blockwise isotropic
scaling-rotation operator, not a softmax-type simplex projection. Its norm
is controlled by the spectral radii \(\rho_i\) and the lift/readout weights.
Afterr the LM-head bottleneck, Koopman-LM replaces repeated softmax covariance
factors, which can collapse under attention saturation, with linear operators
whose singular values are governed by explicit conditioning parameters.
This provides a plausible mechanism for better gradient utilization. 
\newpage
\section{Figures}
\begin{figure}
    \centering   \includegraphics[width=1.0\linewidth]{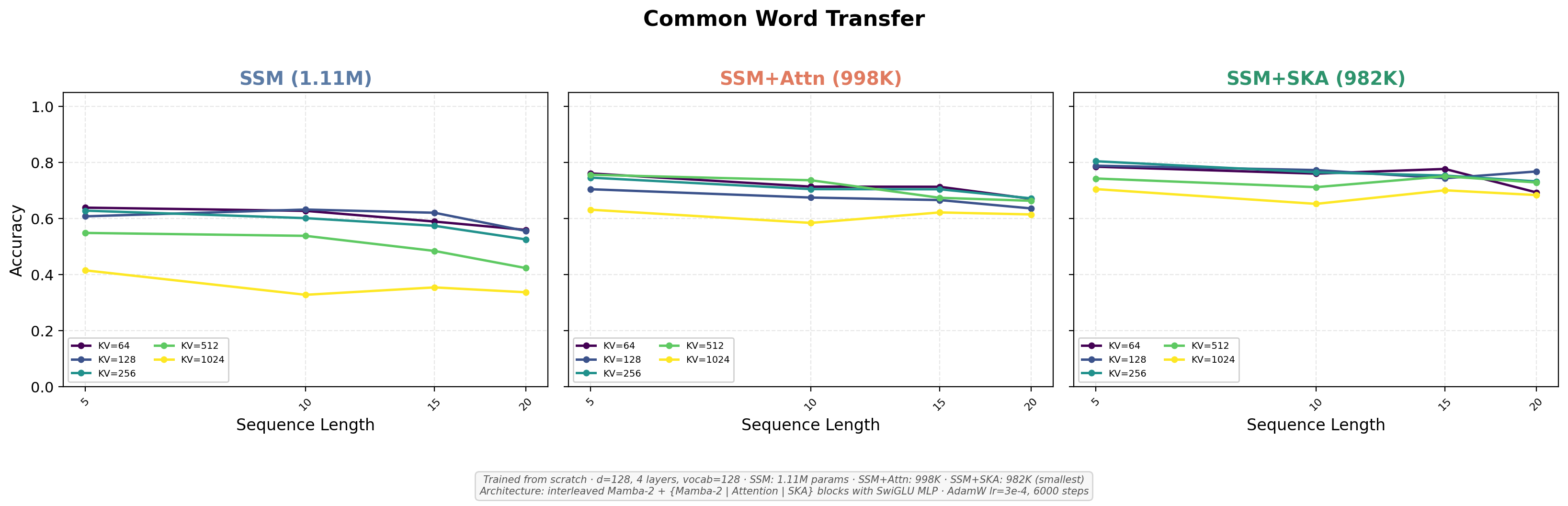}
    \caption{Common word transfer results}
    \label{fig:Common_word}
\end{figure}

\begin{figure}
    \centering
    \includegraphics[width=1.0\linewidth]{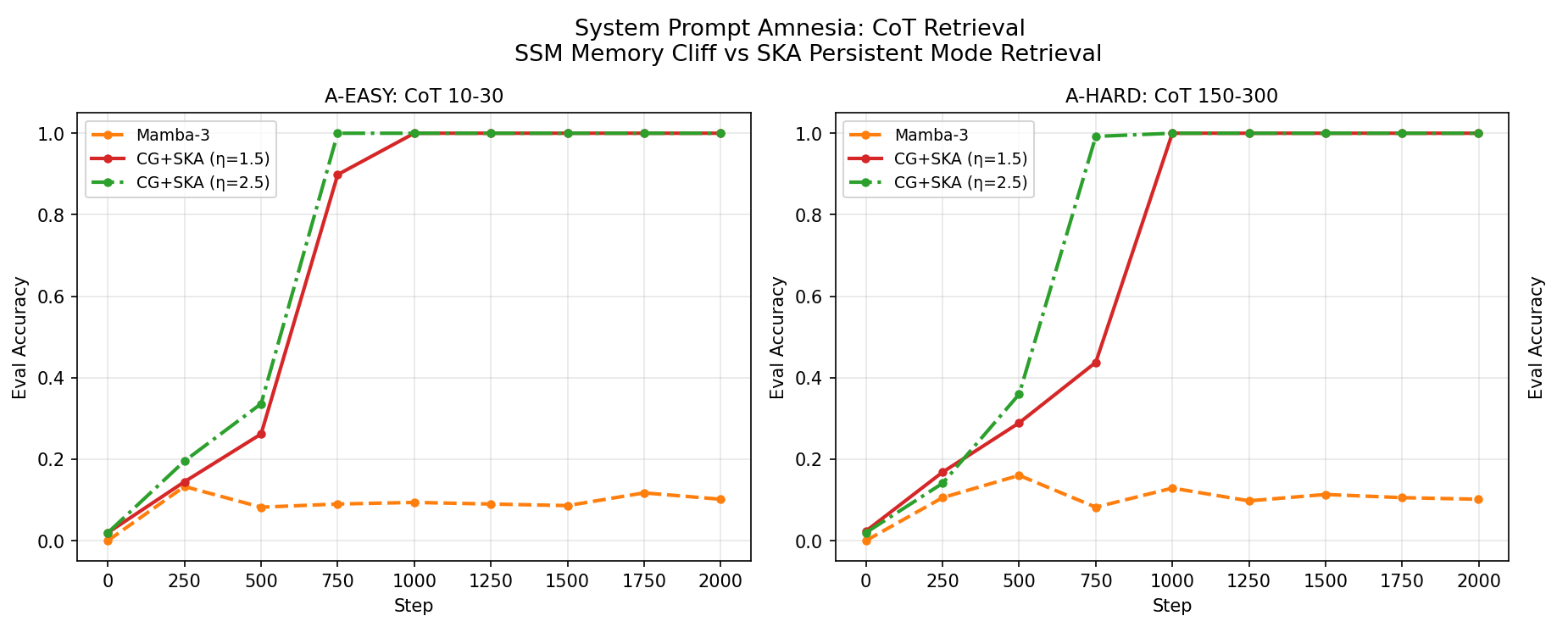}
    \caption{Cot Experiment, synthetic}
    \label{fig:COT_exp}
\end{figure}
\begin{figure}
    \centering
    \includegraphics[width=1.0\linewidth]{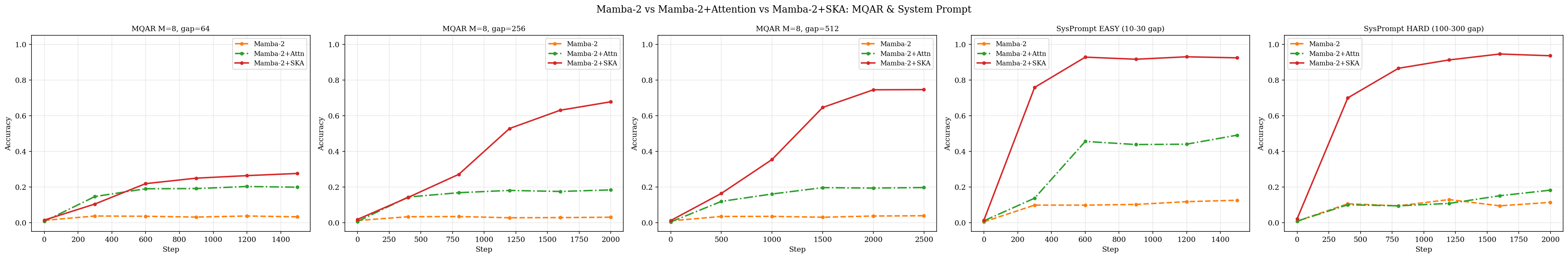}
    \caption{Enter Caption}
    \label{fig:Mamba_2_benchmark}
\end{figure}
\begin{figure}
    \centering
    \includegraphics[width=1.0\linewidth]{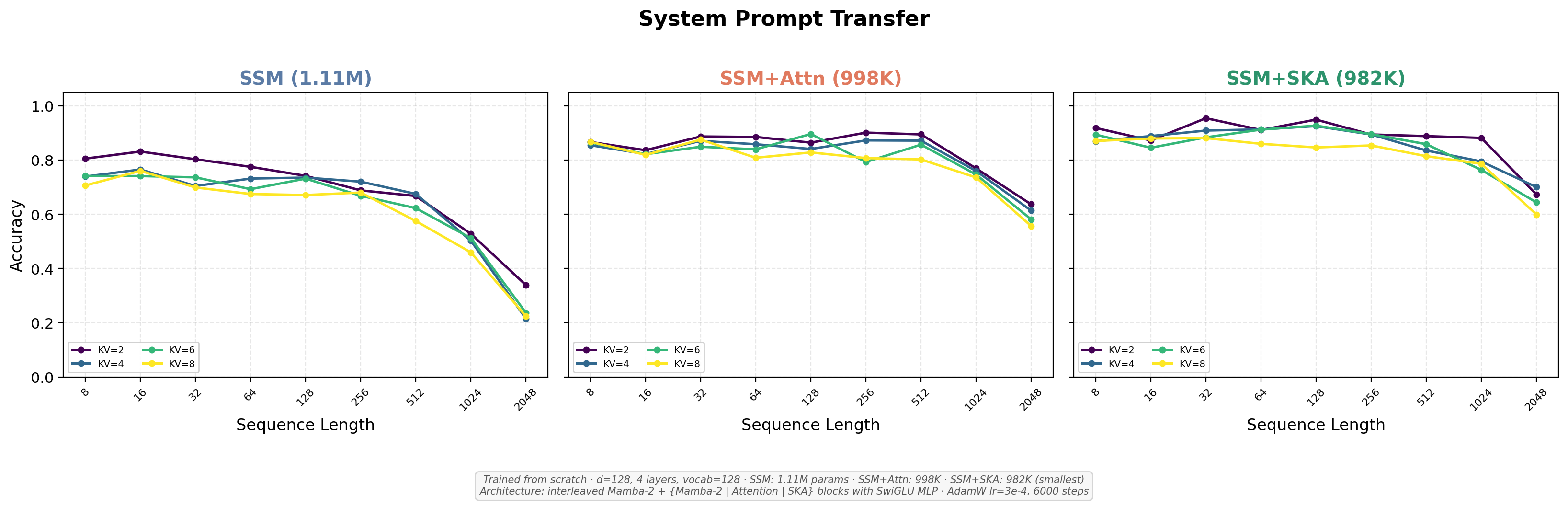}
    \caption{Sysprompt transfer}
    \label{fig:Sysprompt}
\end{figure}

\end{document}